\documentclass[10pt,final,journal]{IEEEtran}

\usepackage{graphics} 
\usepackage{epsfig} 
\usepackage{times} 
\usepackage{amsmath} 
\usepackage{amssymb}  
\usepackage{microtype}
\usepackage{graphicx}
\usepackage{subfig}
\usepackage{booktabs} 
\usepackage{amsfonts}
\usepackage{mathrsfs}
\usepackage{doi}
\usepackage{bm}
\usepackage{float}
\usepackage{hyperref}
\usepackage{cite}
\usepackage{amsthm,amssymb}
\usepackage{bbm}
\usepackage{chngcntr}
\usepackage{algorithm,multirow,xcolor}
\usepackage{algorithmicx}
\usepackage{algpseudocode}

\newtheorem{theorem}{Theorem}
\newtheorem{lemma}{Lemma}
\newtheorem{remark}{Remark}
\newtheorem{definition}{Definition}
\newtheorem{assumption}{Assumption}

\DeclareMathOperator{\Tr}{Tr}

\newcommand{\define}{\triangleq}
\newcommand{\beq}{\begin{equation}}
\newcommand{\eeq}{\end{equation}}
\newcommand{\bea}{\begin{eqnarray}}
\newcommand{\eea}{\end{eqnarray}}

\newcommand{\R}{{\ensuremath{\mathbb{R}}}}

\title{\LARGE \bf
	Lyapunov-Based Reinforcement Learning \\State Estimator}

\begin{document}
	
	
	\author{Liang Hu$^{1}$,  Chengwei Wu$^{2}$, and Wei Pan$^{3}$
		\thanks{$^{1}$Liang Hu is with the School of Computer Science and Electronic Engineering, University of Essex, UK.
			\texttt{\small Email: l.hu@essex.ac.uk}.}%
		\thanks{$^{2}$Chengwei Wu is with the School of Astronautics, Harbin Institute of Technology, China. He is also with the Department of Cognitive Robotics, Delft University of Technology, Netherlands. \texttt{\small Email: c.wu-1@tudelft.nl}.}
		\thanks{$^{3}$Wei Pan is with the Department of Cognitive Robotics, Delft University of Technology, Netherlands. \texttt{\small Email: wei.pan@tudelft.nl}.}%
	}
	\maketitle
	
	\begin{abstract}
		In this paper, we consider the state estimation problem for nonlinear discrete-time stochastic systems. We combine Lyapunov's method in control theory and deep reinforcement learning to design the state estimator. We theoretically prove the convergence of the bounded estimate error solely using the data simulated from the model. An actor-critic reinforcement learning algorithm is proposed to learn the state estimator approximated by a deep neural network. The convergence of the algorithm is analysed. The proposed Lyapunov-based reinforcement learning state estimator is compared with a number of existing nonlinear filtering methods through Monte Carlo simulations, showing its advantage in estimating convergence even under some system uncertainties such as covariance shift in system noise and randomly missing measurements. This is the first reinforcement learning-based nonlinear state estimator with bounded estimate error performance guarantee to the best of our knowledge.
	\end{abstract}
	

	\section{Introduction}
	State estimation, inferring unknown states using noisy measurements and underlying model of the dynamic system, has found its important application in different areas ranging from control engineering, robotics, tracking and navigation to machine learning \cite{anderson2012optimal,thrun2005probabilistic,bar2004estimation,sarkka2013bayesian}. For linear stochastic systems with Gaussian noise, the renowned Kalman filter \cite{kalman1960new} is proved to give the optimal estimate with elegant convergence properties in estimate error covariance. For nonlinear stochastic systems, since the probability distribution of the state does not preserve the property of Gaussian, a closed-form estimator like Kalman filter generally does not exist. As a result, many nonlinear state estimation methods based on different approximation techniques have been proposed, among which the extended Kalman filter (EKF), unscented Kalman filter (UKF) \cite{julier2004unscented} and particle filter (PF) \cite{doucet2000sequential} have been most widely used. For nonlinear systems with Gaussian noise, the EKF and UKF approximate the system state by a Gaussian distribution using the linearisation technique and the deterministic sampling method. For nonlinear non-Gaussian systems, the PF uses sequential Monte Carlo methods to approximate distribution of the state by a finite number of particles.

	A fundamental problem of state estimation is how to design a state estimator such that the estimation error can be guaranteed to be bounded. This problem is challenging and scarcely addressed only in certain scenarios under some assumptions \cite{boutayeb1997convergence,reif1999stochastic,li2012stochastic}. In \cite{boutayeb1997convergence}, it is proved that if the linearisation error is negligible, the local asymptotic stability of the EKF can be ensured. In \cite{reif1999stochastic}, the estimate error of the discrete-time EKF remains bounded only if the initial estimate error is sufficiently small and system noise intensity is small. In \cite{li2012stochastic}, similar results can also be found for the UKF. Nonetheless, it is hard to reduce the initial estimate errors and linearisation errors in the deployment of the EKF and UKF, making these convergence conditions less applicable.
	For the PF, though it is theoretically proved that the estimate performance converges asymptotically to the optimal estimates as the number of particles goes to infinity in \cite{crisan2002survey}, only a limited particle can be used in practice due to the expensive online computation as the number of particle increases. Unfortunately, there are not any guarantees on estimation performance for the PF with limited particles. By summarising the EKF, UKF and PF observations, it is expected that an ideal state estimator can estimate the state with bounded estimate error performance guarantee for nonlinear non-Gaussian systems. Moreover, the performance guarantee should not rely on conditions of the initial estimate error, the system noise level, or the degree of nonlinearity of the underlying dynamic system. In this paper, we are interested in seeking a reinforcement learning (RL) based method to design such a state estimator.
	
	RL was first applied to the state estimation problem in \cite{morimoto2007reinforcement}, where impressive estimation performance was shown. However, the theoretical guarantee on the convergence of bounded estimate error was unavailable, thus making it less applicable for practical applications \cite{bucsoniu2018reinforcement}. More recently, an RL based Kalman filter design method was proposed in \cite{tsiamis2019sample}, in which the bounded estimate error can be guaranteed by using finite samples for stable linear systems. {Another closely related research is (deep) neural network based state estimation \cite{yadaiah2006neural, wilson2009neural,krishnan2015deep} where neural networks are used to learn system models. Recurrent neural networks (RNNs) for state estimation
		has been explored in \cite{yadaiah2006neural, wilson2009neural}, and DNNs has been introduced for Kalman filtering in \cite{krishnan2015deep}.
		Inspired by these works, we will combine deep learning and reinforcement learning, a.k.a., deep reinforcement learning (DRL), to design a state estimator in which the filter gain function will be learned using a deep neural network (DNN) from the sequence of estimate errors. Once the estimator is trained well offline, it can be deployed online and supposed to be efficient given the advance in DNN microprocessor for online applications. The key questions are how to prove the bounded estimate error guarantee in an RL setup and design an efficient learning algorithm with such a guarantee.}

	Motivated by \cite{reif1999stochastic,li2012stochastic,boutayeb1997convergence}, we plan to prove the stability of the estimate error dynamics; after that, the bounded estimate error can be guaranteed. In particular, we make no assumptions typically used in these results, such as initial estimate error, linearisation errors or assumptions on model nonlinearities, etc. Similar to \cite{morimoto2007reinforcement}, the mathematical model can be seen as a simulator, and the training of estimator is solely from data simulated using the mathematical model. Unfortunately, the data-based stability analysis of closed-loop systems in a model-free RL manner is still an open problem \cite{bucsoniu2018reinforcement,gorges2017relations}.
	Typically, Lyapunov's method in control theory is widely used to analyse the stability of dynamical systems. In \cite{postoyan2017stability}, the stability of a deterministic nonlinear system was analysed using Lyapunov's method, assuming that the discount of the infinite-horizon cost is sufficiently close to $1$. However, such an assumption makes it difficult to guarantee the learned policy's optimality without introducing certain extra assumptions \cite{murray2003adaptive,abu2005nearly,jiang2015global}. As a basic tool in control theory, the construction/learning of the Lyapunov function is not trivial \cite{prokhorov1994lyapunov,prokhorov1999application}. In~\cite{perkins2002lyapunov}, the RL agent controls the switch between designed controllers using Lyapunov domain knowledge so that any policy is safe and reliable. \cite{petridis2006construction} proposes a straightforward approach to construct the Lyapunov functions for nonlinear systems using DNNs. In \cite{richards2018lyapunov,berkenkamp2017safe}, a learning-based approach for constructing Lyapunov neural networks to ensure stability was proposed based on the assumption that the learned model is a Gaussian process, Lipschitz continuous and on discretised points in the subset of low-dimensional state space. Only until recently, the asymptotic stability in model-free RL is given for robotic control tasks \cite{han2020actor}. In \cite{zhang2020model-1, zhang2020model-2}, the stability of a system with a combination of a classic baseline controller and an RL controller is proved for autonomous surface vehicles with collisions.

	In summary, we will combine Lyapunov's method in control theory and deep reinforcement learning to design state estimators for nonlinear stochastic discrete-time systems with bounded estimate error guarantee.
	The theoretical result is obtained by solely using the data simulated from the model. In our method, the estimator is trained/learned offline and then deployed directly using a DNN. Moreover, our state estimator is shown to be robust to system uncertainties, i.e., unknown measurement noise, missing measurement and non-Gaussian noise, which makes our method more applicable.
	The main contribution of the paper has twofold:
	\begin{enumerate}
		\item For the first time, a deep reinforcement learning method has been employed for nonlinear state estimator design;
		\item The bounded estimate error can be theoretically guaranteed by solely using data regardless of the degree of model nonlinearities and noise distribution.
	\end{enumerate}
	This is the first reinforcement learning-based nonlinear state estimator design method with bounded estimate error guarantee to the best of our knowledge.

	The rest of the paper is organised as follows. In Section \ref{sec: problemformulation}, the state estimation problem of nonlinear stochastic discrete-time systems is formulated. In Section \ref{sec:theory}, the theoretical result on bounded estimate error guarantee is proved. In Section \ref{sec:algorithm}, the learning algorithm is derived and the algorithm convergence is analysed. In Section \ref{sec:Sim}, our method is compared with EKF, UKF and PF in simulations. Conclusion is given in Section \ref{sec:conclusion}.

	\emph{Notation:} The notation used here is fairly standard except where otherwise stated. $\mathcal{Z}_{+}$ denotes the set of non-negative integrals.  ${\mathbb R}^{n}$ and ${\mathbb R}^{n \times m}$ denote, respectively, the $n$ dimensional Euclidean space and the set of all $n \times m$ real matrices. $A^{\top}$ represents the transpose of $A$, and ${\mathbb E}\{x\}$ stands for the expectation of the stochastic variable $x$. $x \sim \mathcal{N} (m, N)$ with $m \in {\mathbb R}^{n}$ and $N \in {\mathbb R}^{n \times n}$ denotes the probability function of the random variable $x$ follows a Gaussian distribution with $m$ and $N$ as the expectation and covariance, respectively. $\|x\|$ denotes 2-norm of the vector x, i.e., $\|x\|=x^{\top}x$.

	\section{Problem Formulation}\label{sec: problemformulation}
	In this paper, we consider the state estimation problem for the following nonlinear discrete-time  stochastic systems:
	\begin{equation}\label{equ01}
	\begin{split}
	x_{k+1}&=f(x_{k})+w_{k},\\
	y_{k}&=g(x_{k})+v_{k},
	\end{split}
	\end{equation}
	where $x_{k} \in \R^{n}, y_{k} \in \R^m$ are the state and measurement, respectively, and $w_{k}$ and $v_{k}$ are stationary stochastic noise. The nonlinear functions $f(\cdot)$ and $g(\cdot)$ and the  probability distributions of process noise $w_{k}$ and measurement noise $v_{k}$ are assumed to be all known, and the noise may be non-Gaussian.
	
	\subsection{State estimation}
	To estimate the state of system \eqref{equ01}, the state estimator is typically designed in the following form:
	\begin{equation}\label{equ02}
	\hat{x}_{k+1}=f(\hat{x}_{k})+a(\hat{x}_{k})e_{k+1}
	\end{equation}
	where $\hat{x}_{k}$ is the estimate of state $x_{k}$, $a(\cdot)$ is a linear/nonlinear function that can be calculated using various approximation methods and the measurement prediction error is given as follows:
	\begin{equation}\label{equ03}
	e_{k+1}=y_{k+1}-g(f(\hat{x}_{k}))
	\end{equation}
	
	The form of the state estimator in \eqref{equ02} is standard and widely used in some existing estimation algorithms, such as the Kalman filter (KF) and extended Kalman filter (EKF). In the EKF, the state estimator of \eqref{equ01} is given as follows:
	\begin{equation}
	\begin{split}
	\hat{x}_{k+1}=f(\hat{x}_{k})+K_k e_{k+1}
	\end{split}
	\end{equation}
	where $K_k$ is the estimator gain at time instant $k$ that is calculated using partial derivatives of the $f$ and $g$ at $\hat{x}_{k}$. As a result, the estimator gain $K_k$ is actually a function of $\hat{x}_{k}$.
	Similarly, in the unscented Kalman filter (UKF), $a(\hat{x}_{k})$ in \eqref{equ02} is instantiated using deterministic sampling methods. As for the particle filter (PF), even though no estimator gain is used explicitly, the importance weights of particles could be viewed as the function of $a(\hat{x}_{k})$ in the sampling form. In this paper, our high-level plan is to approximate $a(\cdot)$ as a generic nonlinear function, i.e., a deep neural network, which can be learned from estimate errors data $x_{k}-\hat{x}_{k}$ over time.

	\begin{remark}
		This paper will focus on the state estimation problem while two other closely related problems need to be remarked, i.e., state prediction and smoothing. In the state-prediction, measurement up to the current time instant is used to predict the state in the future; in the state-smoothing, measurement up to the current time instant is used to interpolate the state in the past. Our method proposed can be easily adapted to state prediction and smoothing problems as well.
	\end{remark}
	
	\begin{definition}{\cite{reif1999stochastic} }
		The estimate error $\tilde{x}_{k}\define x_{k}-\hat{x}_{k}$ is said to be exponentially bounded in mean square if $\exists~\eta>0$ and $0<\varphi<1$, such that
		\begin{equation}\label{covergence}
		\mathbb{E}[\|\tilde{x}_{k}\|^2]\leq \eta\mathbb{E}[\|\tilde{x}_{0}\|^2]\varphi^{k}+p,
		\end{equation}
		holds at all the time instants $k\geq 0$, where $p$ is a positive constant number.
	\end{definition}

	In this paper, we aim to learn the state estimator policy function $a(\cdot)$ in the estimator \eqref{equ02} using deep reinforcement learning such that the estimate error of the estimator \eqref{equ02} is guaranteed to converge exponentially to a positive bound in the mean square, as defined in Definition 1. { In the following, we will introduce the background on reinforcement learning and show how to analyse the stability using Lypapunov's method.}

	\section{Reinforcement Learning and Data-Based Lyapunov Stability Analysis}\label{sec:theory}
	
	{
		
		In this section, the filtering error dynamics is described as a Markov decision process. Then, some preliminaries of reinforcement learning are presented. Next, a theorem is proposed to prove the boundness of estimate error. 
		
		\subsection{Markov Decision Process}
		
		The estimate error dynamics of the state estimator design \eqref{equ02} can be described by a Markov decision process (MDP), which is defined as a tuple
		$<\mathcal{S},~\mathcal{A},~\mathcal{P},~\mathcal{C},~\gamma >$. Here, $\mathcal{S}$ is the state space, $\mathcal{A}$ is the action space, $\mathcal{P}$ is the transition probability distribution, $\mathcal{C}$ is the cost\footnote{We will use cost instead of reward in this paper which is often used in control literature. Maximisation in RL setup will be minimisation instead.}, and $\gamma\in [0,~1)$ is the discount factor. Then we have
		\begin{equation}\label{MDP}
		\tilde{x}_{k+1} \sim \mathcal{P}\left(\tilde{x}_{k+1} | \tilde{x}_{k}, a_{k}\right), \forall k \in \mathbb{Z}_{+},
		\end{equation}
		where $\tilde{x}_{k}$ and $\mathcal{C}_{k}$ are the state and cost at time instant $k$ and $\tilde{x}_{k} \in \mathcal{S}$. An action $a(\cdot)$ at the state  $\tilde{x}(\cdot)$ is sampled from the policy  $\pi(a(\cdot)|\tilde{x}(\cdot))$. The standard state estimator \eqref{equ02} is naturally a special case of \eqref{MDP}.
		
		\subsection{Reinforcement learning} \label{subsec:RL}
		
		In this paper, the objective is to find an optimal policy to minimise the expected accumulated cost as a value function:
		\begin{equation}
		V_{{\pi}}\left(\tilde{x}_{k}\right)=\sum_{k}^{\infty}\sum_{a_{k}}{\pi}\left(a_{k}|{x}_{k}\right)\sum_{\tilde{x}_{k+1}}\mathcal{P}_{k+1|k}\big(\mathcal{C}_k+\gamma V_{{\pi}}(\tilde{x}_{k+1}) \big),
		\label{eq:V_Func}
		\end{equation}
		where $\mathcal{P}_{k+1|k}=\mathcal{P}\left(\tilde{x}_{k+1}\left|\tilde{x}_{k},a_{k}\right.\right)$ is the transition probability of the estimate error $\tilde{x}_k$, $\mathcal{C}_{k}=\mathcal{C}(\tilde{x}_{k},a_{k})$ is the cost function where we are interested in the quadratic form $\mathcal{C}_{k}=\tilde{x}^{\top}_{k}\tilde{x}_{k}$ in this paper, $\gamma\in\left[0,1\right)$ is a constant discount factor, and ${\pi}\left(a_{k}|\tilde{x}_{k}\right)$ is a policy to be learned. In RL, the policy (nonlinear filter gain function in this paper) ${\pi}\left(a_{k}|\tilde{x}_{k}\right)$ is typically a Gaussian distribution: 
		\begin{equation}
		{\pi}\left(a|\tilde{x}\right)=\mathcal{N}\left(a\left(\tilde{x}\right), {\sigma}\right), \label{eq:Policy}
		\end{equation}
		from which an action $a_{k}\in\mathcal{U}$ at the state $\tilde{x}_{k}\in\mathcal{S}$ is sampled \cite{sutton2018reinforcement}. During the inference, the mean value ${a}\left(\tilde{x}\right)$ is applied.

		During the training process, a Q-function $Q_{{\pi}}\left(\tilde{x}_{k},a_{k}\right)$ (i.e., the action-value function) is practically minimised. $Q_{{\pi}}\left(\tilde{x}_{k},a_{k}\right)$ is given as
		\begin{equation}
		Q_{{\pi}}\left(\tilde{x}_{k},a_{k}\right)=\mathcal{C}_k+\gamma \mathbb{E}_{\tilde{x}_{k+1}}\left[V_{{\pi}}(\tilde{x}_{k+1}) \right],\label{eq:Action-Value Func}
		\end{equation}
		where $\mathbb{E}_{\tilde{x}_{k+1}}\left[\cdot\right]=\sum_{\tilde{x}_{k+1}}\mathcal{P}_{k+1|k}\left[\cdot\right]$ is an expectation operator over the distribution of $\tilde{x}_{k+1}$.
		
		Soft actor-critic (SAC) algorithm is one of the state-of-the-art off-policy actor-critic RL algorithms \cite{Haarnoja2018SAC1}. In SAC, an entropy item is added to the Q-function as a regulariser, with which the exploration performance becomes adjustable. Based on \eqref{eq:Action-Value Func}, the Q-function with the entropy item is described as
		\begin{align}
		Q_{\pi}\left(\tilde{x}_{k},a_{k}\right)=& \mathcal{C}_k+\gamma \mathbb{E}_{\tilde{x}_{k+1}}\left[V_{\pi}(\tilde{x}_{k+1}) \right.\nonumber \\
		&\left.-\alpha\mathcal{H}\left(\pi\left(a_{k+1}|\tilde{x}_{k+1}\right)\right)\right],\label{eq:Q_SAC}
		\end{align}
		where $\mathcal{H}\left({\pi}\left(a_{k}|\tilde{x}_{k}\right)\right)=-\sum_{a_{k}}{\pi}\left(a_{k}|\tilde{x}_{k}\right)\ln\left({\pi}\left(a_{k}|\tilde{x}_{k}\right)\right)=-\mathbb{E}_{{\pi}}\left[\ln\left({\pi}\left(a_{k}|\tilde{x}_{k}\right)\right)\right]$ is the entropy of the policy, and $\alpha$ is a temperature parameter. It is defined to determine the relative importance of the entropy term \cite{Haarnoja2018SAC1}.
		
		Thus, the algorithm is to solve the following optimisation problem.
		\begin{align}
		{\pi}^{\ast}=& \arg\min_{{\pi}\in \Pi} \left(\mathcal{C}_k+\gamma \mathbb{E}_{\tilde{x}_{k+1}}\left[V_{{\pi}}(\tilde{x}_{k+1}) \right.\right.\nonumber \\
		&\big.\left.-\alpha\mathcal{H}\left({\pi}\left(a_{k+1}|\tilde{x}_{k+1}\right)\right)\right]\big),\label{eq:LAC_Obj}
		\end{align}
		where $\Pi$ is the policy set.
		
		An optimal policy $\pi^{\ast}\left(a|\tilde{x}\right)=\mathcal{N}\left(a^{\ast}\left(\tilde{x}\right), \sigma^{\ast}\right)$ can be obtained by solving \eqref{eq:LAC_Obj}. Here, $\sigma^{\ast}$ is close to $0$, which further implies the optimal policy $a^{\ast}$ converges to a deterministic mean value. Once such a policy is obtained, it can be deployed to the target system. For $a^{\ast}$, it can be parameterised as a DNN and learned using stochastic gradient descent algorithms.

		Training/learning process will repeatedly execute policy evaluation and policy improvement. In the policy evaluation, the Q-value in (\ref{eq:Q_SAC}) is computed by applying a Bellman operation $Q_{{\pi}}\left({x}_{k},a_{l, k}\right)=\mathcal{T}^{{\pi}}Q_{{\pi}}\left({x}_{k},a_{l, k}\right)$ where
		\begin{align}
		\mathcal{T}^{{\pi}}Q_{{\pi}}\left({x}_{k},a_{l, k}\right)&=\mathcal{C}_k+\gamma \mathbb{E}_{{x}_{k+1}}\left[\mathbb{E}_{\pi}\left[Q_{\pi}\left({x}_{k+1},a_{k+1}\right)\right]\right], \label{eq:BellmanOp}
		\end{align}
		where $Q_{{\pi}}\left({x}_{k},a_{l,k}\right)=Q_{{\pi}}\left({x}_{k},a_{k}\right) +\alpha\ln\left({\pi}\left(a_{k}|{x}_{k}\right)\right)$.
		
		In the policy improvement, the policy is updated by
		\begin{equation}
		{\pi}_{\text{new}} = \arg \min_{{\pi}'\in \Pi}\mathscr{D}_{\text{KL}}\left({\pi}'\left(\cdot\vert{x}_{k}\right) \Big\Vert \frac{{e^{\frac{-1}{\alpha}Q^{{{\pi}}_{\text{old}}}\left({x}_{k}, \cdot\right)}}}{{Z^{{{\pi}}_{\text{old}}}}}\right), \label{eq:KL_pi}
		\end{equation}
		where ${\pi}_{\text{old}}$ denotes the policy from the last update, $Q^{{{\pi}}_{\text{old}}}$ is the Q-value of ${\pi}_{\text{old}}$, $\mathscr{D}_{\text{KL}}$ denotes the Kullback-Leibler (KL) divergence, and $Z^{{\pi}_{\text{old}}}$ is a normalisation factor. The objective can be transformed into
		\begin{equation}
		{\pi}^* = \arg \min_{{\pi}\in \Pi} \mathbb{E}_{{{\pi}}}\Big[\alpha\ln\left({\pi}\left(a_{k}|{x}_{k}\right)\right)+Q\left({x}_{k}, a_{k}\right)\Big]. \label{eq:PI_Q}
		\end{equation}
		More details can be found in \cite{Haarnoja2018SAC1}. }

	\subsection{Data-based Stability Analysis}\label{sec:DSA}
	
	The convergence of bounded estimate error $\tilde{x}_k$ is essentially equivalent to show the stability of \eqref{MDP}. In this paper, we are interested in establishing the stability theorem by only using samples $\{\tilde{x}_{k+1}, \tilde{x}_{k}, a_k\}, ~\forall k$. The most useful and general approach for studying a dynamical system's stability is the Lyapunov method \cite{lyapunov1892general,jiang2012computational,lewis2012optimal}. In Lyapunov method, a suitable ``energy-like'' Lyapunov function $\mathcal{L}(\tilde{x}_{k})$ (for succinctness $\mathcal{L}(k)$ is used in the following context) is selected. Its derivative along the system trajectories is ensured to be negative semi-definite, i.e., $\mathcal{L}(k+1) - \mathcal{L}(k)\leq 0$ for all time instants and states, so that the state goes in the direction of decreasing the value of Lyapunov function and eventually converges to the origin or a sub-level set of the Lyapunov function. In this subsection, such a function is introduced to analyse the stability of the estimate error dynamical system \eqref{MDP}. Unlike the model-based stability analysis literature, we will learn a Lyapunov function instead of an explicit expression regardless of the degree of nonlinearity of the system and the time-consuming human expert design. {The Lyapunov candidate $\mathcal{L}(k)$ can be chosen as the Q-function $Q_{\pi}(\tilde{x}_{k},a_{k})$ \cite{perkins2002lyapunov,petridis2006construction,richards2018lyapunov, berkenkamp2017safe}. It is well-known that the Q-function $Q_{\pi}(\tilde{x}_{k},a_{k})$ is related to the cost $\mathcal{C}_{k}$.
		Thus, there always exist a positive function
		\begin{equation}
		\begin{aligned}
		\mathcal{L}(k) = \mathcal{C}_{k} + \delta_{k} = Q_{\pi}(\tilde{x}_{k},a_{k})
		\label{c+delta}
		\end{aligned}
		\end{equation}
		where $Q_{\pi}(\tilde{x}_{k},a_{k})$ can be parameterised by a DNN. Here, $\delta_{k} > 0$ should be satisfied due to the property of the Q-function, and it should be also non-increasing. Once the trace of the covariance of estimate error converges, the cost $\mathcal{C}_{k}$ and Q-function both converge as well, as a result $\delta_{k}$ will converge to a constant number.} We first prove the stability theorem by exploiting the function $\mathcal{L}(k)$ and the characteristic of $\delta_{k}$. The details of the learning algorithm will be discussed in Section \ref{sec:algorithm}. { We need the following assumption and lemma for the proof.
		\begin{assumption}
			\label{assumption:ergodicity}
			A Markov chain induced by a policy $\pi$ is ergodic with a unique distribution probability $q_{\pi}(\tilde{x}_{k})$ with $q_{\pi}(\tilde{x}_{k}) = \lim _{k \rightarrow \infty} \mathcal{P}(\tilde{x}_{k} \mid \rho, \pi, k)$, where $\rho$ denotes the distribution of starting states.
		\end{assumption}
	}
	{
		\begin{remark}
			The verification of ergodicity is, in general, an open question and also in the long history of ergodic theory. Many systems have proved to be ergodic in physics, statistic mechanics, economics \cite{moore2015ergodic, peters2019ergodicity}. The study of ergodicity of various systems and its verification composes a significant branch of mathematics. If the transition probability is known for all states, the validation is possible but requires an extensive resource of computation power to enumerate through the state space. The existence of the stationary state distribution is generally assumed to hold in the RL literature \cite{sutton2009convergent,bhandari2018finite, meyn2012markov}. In this paper, based on the ergodicity assumption, we focus on analysing the stability of such systems and developing an algorithm to find the filter gain. 
		\end{remark}
	}
	{Before proving the main theorem, we also need the following Lemma}
	{
		\begin{lemma}\label{lemma1} \cite{royden1968real}
			(Lebesgue's Dominated convergence theorem) Suppose $f_n:\mathcal{R}\rightarrow [-\infty,+\infty]$ (Lebesgue) measurable functions such that the point-wise limit $f(x)=\lim_{n\rightarrow\infty}f(x)$ exists. Assume there is an integrable function $g:\mathcal{R}\rightarrow [-\infty,+\infty]$ with $|f_n(x)|\leq g(x)$ for each $x\in \mathcal{R}$, then $f$ is integrable (in the Lebesgue sense) and $$\lim_{n\rightarrow\infty}\int_{\mathcal{R}}f_n\mathrm{d}\mu=\int_{\mathcal{R}}\lim_{n\rightarrow\infty}f_n\mathrm{d}\mu=\int_{\mathcal{R}}f\mathrm{d}\mu.$$
		\end{lemma}
		
		{Based on Assumption~\ref{assumption:ergodicity} and Lemma~\ref{lemma1}, } we prove our main theorem in the following.
		\begin{theorem}\label{theo2}
			If there exist a function $\mathcal{L}(k) \geq 0$, and constants $\alpha_{1}\geq0$, $\alpha_{2}>0$, and $\beta \geq 0$ such that the following inequalities hold
			\begin{equation}
			\begin{aligned}
			&\alpha_{1}\mathcal{C}_{k} \leq \mathcal{L}(k) \leq \alpha_{2}\mathcal{C}_{k}, \\
			&\mathbb{E}_{\tilde{x}_{k} \sim \mu_{\pi}}\left[\mathbb{E}_{\tilde{x}_{k+1} \sim \mathcal{P}_{\pi}} \left[\mathcal{L}(k+1) \right]-\mathcal{L}(k)\right] \\
			&\leq - \beta \mathbb{E}_{\tilde{x}_{k} \sim \mu_{\pi}} \left\{\|\tilde{x}_{k}\|^{2} \right\} + \delta_{k}, \ \forall k \in \mathcal{Z}_{+}.
			\label{LyaCons}
			\end{aligned}
			\end{equation}
			Then the estimate error $\tilde{x}_{k}$ is guaranteed to be exponentially bounded in mean square, i.e.,
			\begin{equation}
			\begin{aligned}
			\mathbb{E}_{\tilde{x}_{k} \sim \mu_{\pi}}\left[\left\|\tilde{x}_{k} \right\|^{2}\right] \leq \sigma^{k}\mathbb{E}_{\tilde{x}(0) \sim \mu_{\pi}}\left[\left\|\tilde{x}_{0} \right\|^{2}\right]+ p,
			\label{index}
			\end{aligned}
			\end{equation}
			where $$\mu_{\pi}(\tilde{x}_{k}) \triangleq \lim_{N \rightarrow \infty} \frac{1}{N} \sum_{k=0}^{N} \mathcal{P}\left(\tilde{x}_{k} \mid \rho, \pi, k\right)$$ is the estimate error distribution, and $p=\sum^{k-1}_{\iota=0}\sigma^{k-\iota-1}\delta_{\iota}$, $\sigma \in (0,1)$.
		\end{theorem}
		
		\begin{proof}
			The existence of the sampling distribution $\mu_{\pi}(\tilde{x}_{k})$ is guaranteed by the existence of $q_{\pi}(\tilde{x}_{k}).$ Since the sequence $\left\{\mathcal{P}(\tilde{x}_{k} | \rho, \pi, k), k \in \mathbb{Z}_{+}\right\}$ converges to $q_{\pi}(\tilde{x}_{k})$ as $k$ approaches $\infty$, then by the Abelian theorem that if a sequence or function behaves regularly, then some average of it behaves regularly. We have the sequence $\left\{\frac{1}{N} \sum_{k=0}^{N} P(\tilde{x}_{k} \mid \rho, \pi, k), N \in \mathbb{Z}_{+}\right\}$ also converges and $\mu_{\pi}(\tilde{x}_{k})=q_{\pi}(\tilde{x}_{k})$. Then, \eqref{LyaCons} can be rewritten as
			\begin{equation}
			\begin{aligned}
			&\int_{\mathcal{S}} \lim _{N \rightarrow \infty} \frac{1}{N} \sum_{k=0}^{N} P(\tilde{x}_{k} | \rho, \pi, k)  \left(\mathbb{E}_{P_{\pi}\left(\tilde{x}_{k+1} \mid s\right)} \mathcal{L}(k+1)-\mathcal{L}(k)\right) \mathrm{d} \tilde{x}_{k} \\
			&\leq - \beta \mathbb{E}_{\tilde{x}_{k} \sim \mu_{\pi}} \|\tilde{x}_{k}\|^{2} + \delta_{k}
			\nonumber
			\end{aligned}
			\end{equation}
			
			The left-hand-side of the above inequality can be derived as follows:
			\small
			\begin{equation}
			\begin{aligned}
			&\int_{\mathcal{S}} \lim _{N \rightarrow \infty} \frac{1}{N} \sum_{k=0}^{N} P(\tilde{x}_{k} \mid \rho, \pi, k) \left({\int_{\mathcal{S}} P_{\pi}\left(\tilde{x}_{k+1} | \tilde{x}_{k}\right) \mathcal{L}(k+1) \mathrm{d} \tilde{x}_{k+1} - \mathcal{L}(k)}\right)
			\mathrm{d} \tilde{x}_{k}  \\
			&= \int_{\mathcal{S}} \lim_{N \rightarrow \infty} \frac{1}{N} \sum_{k=0}^{N} P(\tilde{x}_{k} \mid \rho, \pi, k)\int_{\mathcal{S}} P_{\pi}\left(\tilde{x}_{k+1}| \tilde{x}_{k}\right) 
			\mathcal{L}(k+1) \mathrm{d}
			\tilde{x}_{k+1}\mathrm{d} \tilde{x}_{k}  \\
			& \ \ \ \ \  - \int_{\mathcal{S}} \lim_{N \rightarrow \infty} \frac{1}{N} \sum_{k=0}^{N} P(\tilde{x}_{k} \mid \rho, \pi, k)\mathcal{L}(k)
			\mathrm{d} \tilde{x}_{k}  \\
			&= \int_{\mathcal{S}} \lim_{N \rightarrow \infty} \frac{1}{N} \int_{\mathcal{S}}\sum_{k=0}^{N} P(\tilde{x}_{k} \mid \rho, \pi, k) P_{\pi}\left(\tilde{x}_{k+1} | \tilde{x}_{k}\right) \mathcal{L}(k+1) \mathrm{d} \tilde{x}_{k}\mathrm{d} \tilde{x}_{k+1} \\
			& \ \ \ \ \  - \int_{\mathcal{S}} \lim _{N \rightarrow \infty} \frac{1}{N} \sum_{k=0}^{N} P(\tilde{x}_{k} \mid \rho, \pi, k)\mathcal{L}(k)
			\mathrm{d} \tilde{x}_{k}  \\
			&= \int_{\mathcal{S}} \lim _{N \rightarrow \infty} \frac{1}{N}\mathcal{L}(k+1) \mathrm{d} \tilde{x}_{k+1} \int_{\mathcal{S}}\sum_{k=0}^{N} P(\tilde{x}_{k} \mid \rho, \pi, k) P_{\pi}\left(\tilde{x}_{k+1} | \tilde{x}_{k}\right) 
			\mathrm{d} \tilde{x}_{k} \\
			& \ \ \ \ \  - \int_{\mathcal{S}} \lim _{N \rightarrow \infty} \frac{1}{N} \sum_{k=0}^{N} P(\tilde{x}_{k} \mid \rho, \pi, k)\mathcal{L}(k)
			\mathrm{d} \tilde{x}_{k}  \\
			&=\int_{\mathcal{S}} \lim _{N \rightarrow \infty} \frac{1}{N} \int_{\mathcal{S}}\sum_{k=0}^{N} P(\tilde{x}_{k+1} \mid \rho, \pi, k+1)\mathcal{L}(k+1)  
			\mathrm{d} \tilde{x}_{k+1} \\
			& \ \ \ \ \  - \int_{\mathcal{S}} \lim _{N \rightarrow \infty} \frac{1}{N} \sum_{k=0}^{N} P(\tilde{x}_{k} \mid \rho, \pi, k)\mathcal{L}(k)
			\mathrm{d} \tilde{x}_{k}
			\label{LHSInq}
			\end{aligned}
			\end{equation}
			\normalsize
			Since $\mathcal{L}(k) \leq \alpha_2\mathcal{C}_k$ and $P(\tilde{x}_{k} | \rho, \pi, k)\leq 1$, then $P(\tilde{x}_{k} | \rho , \pi, k) \mathcal{L}(k)\leq \alpha_2\mathcal{C}_k$. Besides, the sequence $\left\{\frac{1}{N} \sum_{k=0}^{N} \mathcal{P}(\tilde{x}_{k} | \rho, \pi, k) \mathcal{L}(k)\right\}$ converges point-wise to the function $q_{\pi}(\tilde{x}_{k}) \mathcal{L}(k)$. 
			According to the Lebesgue's dominated convergence theorem (Lemma \ref{lemma1}), it follows from \eqref{LHSInq} that :
			\small
			\begin{equation}
			\begin{aligned}
			&\int_{\mathcal{S}} \lim _{N \rightarrow \infty} \frac{1}{N} \sum_{k=0}^{N} P(\tilde{x}_{k} \mid \rho, \pi, k) \left({\int_{\mathcal{S}} P_{\pi}\left(\tilde{x}_{k+1} | \tilde{x}_{k}\right) \mathcal{L}(k+1) \mathrm{d} \tilde{x}_{k+1} - \mathcal{L}(k)}\right)
			\mathrm{d} \tilde{x}_{k}  \\
			&=\int_{\mathcal{S}} \lim _{N \rightarrow \infty} \frac{1}{N} \int_{\mathcal{S}}\sum_{k=0}^{N} P(\tilde{x}_{k+1} \mid \rho, \pi, k+1)\mathcal{L}(k+1)  
			\mathrm{d} \tilde{x}_{k+1} \\
			& \ \ \ \ \  - \int_{\mathcal{S}} \lim _{N \rightarrow \infty} \frac{1}{N} \sum_{k=0}^{N} P(\tilde{x}_{k} \mid \rho, \pi, k)\mathcal{L}(k)
			\mathrm{d} \tilde{x}_{k}  \\
			&= \lim _{N \rightarrow \infty} \frac{1}{N}\left(\sum_{k=1}^{N+1} \mathbb{E}_{P(\tilde{x}_{k} | \rho, \pi, k)} 
			\mathcal{L}(k) -\sum_{k=0}^{N} \mathbb{E}_{P(\tilde{x}_{k} \mid \rho, \pi, k)} \mathcal{L}(k)\right)  \\
			&= \lim _{N \rightarrow \infty} \frac{1}{N}\sum_{k=0}^{N} \left(\mathbb{E}_{P(\tilde{x}_{k+1} | \rho, \pi, k+1)} \mathcal{L}(k+1) - \mathbb{E}_{P(\tilde{x}_{k} \mid \rho, \pi, k)} \mathcal{L}(k)\right).
			\label{ExpInq}
			\end{aligned}
			\end{equation}
			\normalsize
			It further follows from \eqref{LyaCons} and \eqref{ExpInq} that 
			\small
			\begin{equation}
			\begin{aligned}
			&\mathbb{E}_{\mathcal{P}(\tilde{x}_{k+1} | \rho, \pi, k+1)} \mathcal{L}(k+1)- \mathbb{E}_{\mathcal{P}(\tilde{x}_{k} \mid \rho, \pi, k)} \mathcal{L}(k)\\
			\leq & - \beta \mathbb{E}_{\tilde{x}_{k} \sim \mu_{\pi}} \left\{\|\tilde{x}_{k}\|^{2} \right\} + \delta_{k} \label{Ineqnew}
			\end{aligned}
			\end{equation}
			\normalsize
			Given $\alpha_2$ and $\beta$, there always exists a scalar $\sigma$ such that the following equation holds
			\begin{equation}
			\begin{aligned}
			\left(\frac{1}{\sigma}-1\right)\alpha_{2} -\frac{\beta}{\sigma} =0.
			\label{conseqn}
			\end{aligned}
			\end{equation}
			
			
			Using \eqref{Ineqnew} and \eqref{conseqn}, the following inequality can be derived
			\begin{equation*}
			\begin{split}
			&\frac{1}{\sigma^{\iota+1}}\mathbb{E}_{\mathcal{P}(\tilde{x}_{\iota+1} \mid \rho, \pi, \iota+1)} \mathcal{L}(\iota+1) - \frac{1}{\sigma^{\iota}}\mathbb{E}_{\mathcal{P}(\tilde{x}_{\iota} \mid \rho, \pi, \iota)} \mathcal{L}(\iota) \\
			&= \frac{1}{\sigma^{\iota+1}} \left(\mathbb{E}_{\mathcal{P}(\tilde{x}_{\iota+1} | \rho, \pi, \iota+1)} \mathcal{L}(\iota+1)- \mathbb{E}_{\mathcal{P}(\tilde{x}_{\iota} \mid \rho, \pi, \iota)} \mathcal{L}(\iota)\right) \\
			& + \frac{1}{\sigma^{\iota}}\left(\frac{1}{\sigma} - 1 \right)\mathbb{E}_{\mathcal{P}(\tilde{x}_{\iota} \mid \rho, \pi, \iota)} \mathcal{L}(\iota) \\
			&\leq \frac{1}{\sigma^{\iota}} \left(-\frac{\beta}{\sigma} + (\frac{1}{\sigma} - 1)\alpha_{2} \right)\mathbb{E}_{\tilde{x}_{k} \sim \mu_{\pi}} \left\{\|\tilde{x}_{k}\|^{2} \right\} \\
			& + \frac{\delta_{\iota}}{\sigma^{\iota+1}}, \quad \forall \iota\geq 0,
			\end{split}
			\end{equation*}
			which implies
			\begin{equation}
			\begin{aligned}
			&\frac{1}{\sigma^{\iota+1}}\mathbb{E}_{\mathcal{P}(\tilde{x}_{\iota+1} \mid \rho, \pi, \iota+1)} \mathcal{L}(\iota+1) - \frac{1}{\sigma^{\iota}}\mathbb{E}_{\mathcal{P}(\tilde{x}_{\iota} \mid \rho, \pi, \iota)} \mathcal{L}(\iota) \\
			&\leq \frac{\delta_{\iota}}{\sigma^{\iota+1}}.
			\label{*}
			\end{aligned}
			\end{equation}
			
			To sum the above inequality from $\iota=0,1,\ldots,k-1$ yields
			\begin{equation*}
			\begin{split}
			&\frac{1}{\sigma^{k}}\mathbb{E}_{\mathcal{P}(\tilde{x}_{k} \mid \rho, \pi, k)} \mathcal{L}(k) - \mathbb{E}_{\mathcal{P}(\tilde{x}_{0} \mid \rho, \pi, 0)} \mathcal{L}(0) \\
			& \leq \sum^{k-1}_{\iota=0} \frac{\delta_{\iota}}{\sigma^{\iota+1}},
			\end{split}
			\end{equation*}
			which implies
			\begin{equation*}
			\begin{aligned}
			\mathbb{E}_{\tilde{x}_{k} \sim \mu_{\pi}}\left[\left\|\tilde{x}_{k} \right\|^{2}\right] &\leq \sigma^{k}\mathbb{E}_{\tilde{x}_{0} \sim \mu_{\pi}}\left[\left\|\tilde{x}_{0} \right\|^{2}\right]+ p .
			\end{aligned}
			\end{equation*}
			where the scalar constant $p$ is a upper bound of the sequence $\{\sum_{t=0}^{k-1}\sigma^{k-t-1}\delta_t,k\geq 0\}$. Since $\sigma \in(0,~1)$ and $\delta_{k}$ is non-increasing, we can conclude that the upper bound $p$ exists. Thus,  the estimate error is exponentially bounded in  mean square, which completes the proof.
	\end{proof}}

	\section{Reinforcement Learning Filter Design}\label{sec:algorithm}
	
	In this section, we will discuss the reinforcement learning algorithm design and implementation. We will design the algorithm based on the actor-critic RL algorithms which are widely used in continuous control tasks \cite{sutton2018reinforcement}.
	In actor-critic RL, typically DNNs are used to approximate the ``critic'' and the ``actor''. In the following, we will introduce how to incorporate convergence conditions derived in Subsection \ref{sec:DSA}  into such algorithm architectures. Then the convergence of the algorithm will be analysed.
	
	{\subsection{Deep neural networks approximation} \label{subsec:DNN}
		
		DNNs approximation are used due to the continuous state and action spaces in this paper. The DNNs are constructed by fully connected multiple layer perceptrons (MLP), in which the rectified linear unit (ReLU) nonlinearities are chosen as the activation functions \cite{Dahl2013ICASSP}. The ReLU nonlinearities are defined as $  \rho\left(z\right) = \max\left\{z\text{, }0\right\}$ when $z$ is a scalar. Given a vector $z=[z_1\text{,}\ldots,\text{,}z_n]^{\top}\in\mathbb{R}^{n}$, then $\rho\left(z\right)=[\rho\left(z_1\right)\text{,}\ldots\text{,}\rho\left(z_n\right)]^{\top}$. An example of a MLP with two hidden layers is described as
		\begin{equation}
		\underline{\text{MLP}}_{\mathrm{w}}^{2}\left(z\right) =\mathrm{w}_2\Big[ \rho\Big(\mathrm{w}_1 \Big[\rho\Big(\mathrm{w}_0\left[\begin{array}{c}
		z \\
		1
		\end{array}\right]\Big)\text{,}1\Big]^{\top} \Big)^{\top}\text{,} 1\Big]^{\top}, \label{eq:MLP}
		\end{equation}
		where $\left[z^{\top}\text{, }1\right]^{\top}$ is a vector composed of $z$ and a constant bias $1$, the superscript ``$2$'' denotes the total number of hidden layers, the subscript ``$\mathrm{w}$'' denotes the parameter set to be trained in a MLP with $\mathrm{w}=\left\{\mathrm{w}_0\text{, }\mathrm{w}_1\text{, }\mathrm{w}_2\right\}$, and $\mathrm{w}_0$, $\mathrm{w}_1$,  and $\mathrm{w}_2$ are weight matrices with appropriate dimensions.

		If there is a set of inputs $z=\left\{z_1\text{, } \ldots\text{, } z_L\right\}$ for the MLP in  (\ref{eq:MLP}) with $z_1$, $\ldots$, $z_L$ denoting vector signals, we have
		\begin{equation}
		\underline{\text{MLP}}_{\mathrm{w}}^{2}\left(z\right) =\underline{\text{MLP}}_{\mathrm{w}}^{2}\big(\left[z_1^{\top}\text{, }  \ldots\text{, } z_L^{\top}\right]^{\top}\big). \label{eq:MLP_2}
		\end{equation}
		Besides, $\underline{\text{MLP}}_{\mathrm{w}}^{2}\left(z_1\text{, } z_2\right)=\underline{\text{MLP}}_{\mathrm{w}}^{2}\big(\left[z_1^{\top}\text{, } z_2^{\top}\right]^{\top}\big)$ for two vector inputs $z_1$ and $z_2$. If $z_1=\left\{z_{11}\text{, } \ldots\text{, } z_{1L}\right\}$ is a set of vectors, $\underline{\text{MLP}}_{\mathrm{w}}^{2}\left(z_1\text{, } z_2\right)=\underline{\text{MLP}}_{\mathrm{w}}^{2}\big(\left[z_{11}^{\top}\text{, }\ldots\text{, } z_{1L}^{\top}, z_2^{\top}\right]^{\top}\big)$.

		In this paper, the constructed MLPs are used to approximate the ``critic'' $Q_{\pi}(\tilde{x}_{k},a_{k})$ and the ``actor'' $\pi(a_{k}|\tilde{x}_{k})$. We respectively use $\theta$ and $\phi$ to parameterise $Q(\tilde{x}_{k},a_{k})$ and $\pi(a_{k}|\tilde{x}_{k})$, i.e., $Q_{\theta}(\tilde{x}_{k},a_{k})$ and $\pi_{\phi}(a_{k}|\tilde{x}_{k})$. As discussed in Section \ref{sec:DSA}, the Q-function $Q(\tilde{x}_{k},a_{k})$ is regarded as a Lyapunov candidate $\mathcal{L}(k)$ in our paper. Namely, the ``critic'' is the Lyapunov function. In the following context, we replace $Q(\tilde{x}_{k},a_{k})$ with $\mathcal{L}(k)$, and  $Q_{\theta}(\tilde{x}_{k},a_{k})$ with $\mathcal{L}_{\theta}(k)$. The direct output of the constructed MLP may not satisfy the requirements of a Lyapunov function $\mathcal{L}(k)$, for example, $\mathcal{L}(k)>0,~ \forall k\geq 0$, so it is necessary to modify the representation of MLP. Following (\ref{eq:MLP}) and (\ref{eq:MLP_2}), the Lyapunov function approximation $\mathcal{L}_{\theta}(k)$ is chosen as
		\begin{equation}
		\mathcal{L}_{\theta}(k) =\left(\underline{\text{MLP}}_{\theta}^{{K}_{1}}\left(\tilde{x}_{k}\right)\right)^{2}, \label{eq:MLP_Q}
		\end{equation}
		where $\theta =\left\{\theta_0, \ldots, \theta_{{K}_{1}}\right\}$ with $\theta_i$  denoting the weight matrices of proper dimensions, $0\leq i\leq K_1$.  The DNN for $\mathcal{L}_{\theta}$ is illustrated in Fig. \ref{fig:ACNN}.
		
		The filter gain $a_{k}$ is also approximated using a MLP. The approximated filter gain of $a_{k}$ with a parameter set $\phi$ is
		\begin{equation}
		a_{\phi}=\underline{\text{MLP}}_{\phi}^{{K}_{2}}\left(\tilde{x}_{k}\right). \label{eq:MLP_u}
		\end{equation}
		The illustration of $a_{\phi}$ is given in Figure \ref{fig:ACNN}. In the learning setup, there are two outputs for the MLP in \eqref{eq:MLP_u}. One is the control law $a_{\phi}$, the other one is ${\sigma}_{\phi}$ that is the standard deviation (SD) of the exploration noise \cite{Haarnoja2018SAC1}.  According to (\ref{eq:Policy}), the parameterised policy $\pi_{{\phi}}$ in our learning is
		\begin{equation}
		\pi_{{\phi}} =  \mathcal{N}\left(a_{\phi}\left(\tilde{x}\right), \sigma_{\phi}^2\right).    \label{eq:MLP_Pi}
		\end{equation}
		
		\begin{figure*}[tb]
			\centering
			\includegraphics[width=0.88\textwidth]{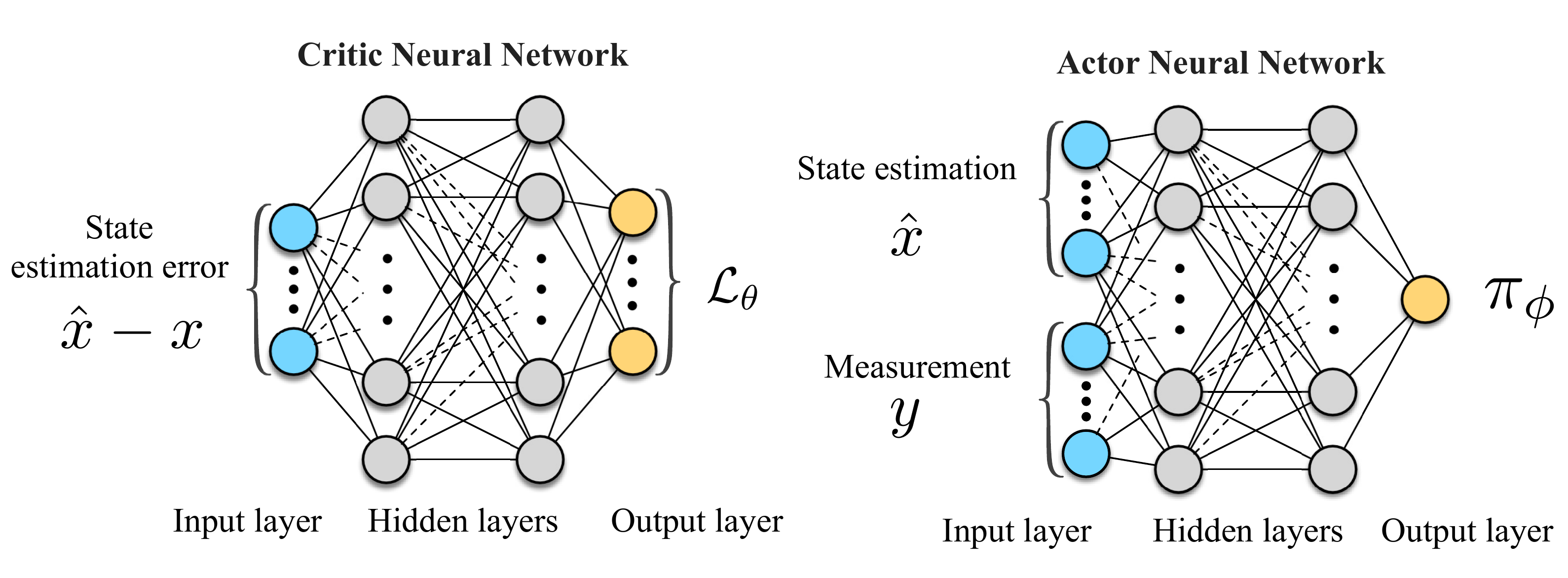}
			\caption{Approximation of $\mathcal{L}_{\theta}$ and $\pi_{\phi}$ using MLP}
			\label{fig:ACNN}
			\vspace{-0.4cm}
	\end{figure*}}
	
	\subsection{Lyapunov reinforcement learning filter (LRLF)}\label{subsec:Training}
	
	The actor-critic RL training process is depicted in Fig.~\ref{fig:TrainRL}. In the training process, the dynamic system \eqref{equ01} and the state estimator to be trained \eqref{equ02} repeatedly run to collect the data, which is restored as the replay memory $\mathcal{M}$. After $\mathcal{M}$ is collected, the policy evaluation and improvement are executed by randomly sampling a batch of data in $\mathcal{M}$. Then, the improved policy $\pi_{\phi}(a_{k}|\tilde{x}_{k})$ is applied to the state estimator to generate data until $\mathcal{L}_{\theta}(k)$ converges.
	
	In the learning stage, the policy is obtained by repeatedly executing the policy evaluation and policy improvement. For the policy evaluation, it starts from any function $\mathcal{L}:$ $\mathcal{S}\times \mathcal{A}\rightarrow R$ under a fixed policy $\pi$, and repeatedly applies a Bellman backup operator $\mathcal{T}^{\pi}$, which is defined as
	\begin{equation}\label{BacOpe}
	\begin{split}
	\mathcal{T}^{\pi}\mathcal{L}_{\pi}(k)
	&= \mathcal{C}_{k}+\gamma\mathbb{E}_{\tilde{x}_{k+1}}\left[\mathbb{E}_{\pi}\left[ \mathcal{L}_{\pi}(k+1) \right] \right].
	\end{split}
	\end{equation}
	
	At each policy evaluation step, the policy $\pi_{\phi}(a_{k}|\tilde{x}_{k})$ should minimise the following Bellman residual equation
	\begin{equation*}
	\mathcal{J}_{\mathcal{L}}(\theta) = \mathbb{E}_{(\tilde{x}_{k},a_{k}\sim \mathcal{M})}\left\{\frac{1}{2}\left(\mathcal{L}_{\theta}(k) - \mathcal{L}_{\text{target}} \right)^{2} \right\},
	\end{equation*}
	where $(\tilde{x}_{k},a_{k}\sim \mathcal{M})$ represents the operation that randomly takes $(\tilde{x}_{k},a_{k})$ from the memory $\mathcal{M}$, and
	\begin{equation}
	\begin{aligned}
	\mathcal{L}_{\text{target}} = \mathcal{C}_{k} + \gamma \mathbb{E}_{\tilde{x}_{k+1}}\left[\mathbb{E}_{\pi}\left[\mathcal{L}_{\bar{\theta}}(k+1)+ \alpha \ln (\pi_{\phi})\right] \right],
	\nonumber
	\end{aligned}
	\end{equation}
	{with $\bar{\theta}$ being the target network parameter.}
	
	We can obtain the following by using stochastic gradient descent:
	\begin{equation*}
	\nabla_{\theta}\mathcal{J}_{\mathcal{L}}(\theta) = \sum \frac{\nabla_{\theta}\mathcal{L}_{\theta}}{|\mathcal{B}|}\left(\mathcal{L}_{\theta}(k) - \mathcal{L}_{\text{target}} \right),
	\end{equation*}
	where $|\mathcal{B}|$ denotes the batch size.
	
	At the policy improvement stage, the improved policy should guarantee the Lyapunov inequality in Theorem \ref{theo2} holds. Thus, the policy is updated as
	\begin{equation}
	\begin{aligned}
	&\pi_{\text{new}} = \arg \min\limits_{\pi'\in \Pi}\mathscr{D}_{\text{KL}}\left(\pi'(\cdot|\tilde{x}_{k})\|\frac{e^{\frac{-1}{\alpha}\mathcal{L}_{\pi_{\text{old}}}}(\tilde{x}_{k},\cdot)}{Z_{\pi_{\text{old}}}(\tilde{x}_{k})} \right) \\
	\text{s.t.}~ &\mathcal{L}_{\theta}(k+1)-\mathcal{L}(k) \leq -\beta \Tr(\tilde{x}_{k}\tilde{x}^{\top}_{k})+\delta_{k},
	\label{PolUpd}
	\end{aligned}
	\end{equation}
	where $\Pi$ is the policy set, $\pi_{\text{old}}$ is the last updated policy, $\mathcal{L}_{\pi_{\text{old}}}$ is the action value function of $\pi_{\text{old}}$, $\mathscr{D}_{\text{KL}}$ means the Kullback-Leibler divergence, and $Z_{\pi_{\text{old}}}$ is a partition function which is introduced to normalise the distribution.
	
	\begin{remark}
		Different from the SAC algorithm, the Lyapunov constraint, that is, $\mathcal{L}_{\theta}(k+1)-\mathcal{L}(k) \leq -\beta \Tr(\tilde{x}_{k}\tilde{x}^{\top}_{k})+\delta_{k}$ is considered when the policy improvement step is executed. In this way, we can guarantee that the estimate error always converges to a positive constant in mean square, which is proved in Theorem \ref{theo2}.
	\end{remark}
	
	Then we can rewrite \eqref{PolUpd} as
	\begin{equation}
	\begin{aligned}
	&\pi_{\text{new}} = \arg \min\limits_{\pi\in \Pi}\mathbb{E}\left[\alpha \ln(\pi(a_{k}|\tilde{x}_{k}))+\mathcal{L}(k) \right] \\
	\text{s.t.}~ &\mathcal{L}_{\theta}(k+1)-\mathcal{L}(k) \leq -\beta \Tr(\tilde{x}_{k}\tilde{x}^{\top}_{k})+\delta_{k}.
	\label{PolUpd-1}
	\end{aligned}
	\end{equation}

	For the optimisation of \eqref{PolUpd-1}, the Lagrangian multiplier can be introduced to deal with the constraint. Thus, \eqref{PolUpd-1} can be further described as
	\begin{equation} \label{PolUpd-2}
	\begin{split}
	\pi_{\text{new}} &= \arg \min\limits_{\pi\in \Pi}\mathbb{E}\left[\alpha \ln(\pi(a_{k}|\tilde{x}_{k}))+\mathcal{L}(k) \right] \\
	&+\lambda\left(\mathcal{L}_{\theta}(k+1)-\mathcal{L}(k) + \beta \Tr(\tilde{x}_{k}\tilde{x}^{\top}_{k})-\delta_{k}\right),
	\end{split}
	\end{equation}
	where $\lambda$ is a Lagrangian multiplier.

	Based on the setup of RL, the policy improvement in \eqref{PolUpd-2} is converted into finding $\pi^{\ast}$ by minimising the following function
	\begin{equation*}
	\mathcal{J}_{\pi}(\phi) = \mathbb{E}_{(\tilde{x}_{k},a_{k}\sim \mathcal{M})}\left\{\alpha \ln(\pi_{\phi}) + \mathcal{L}_{\theta}(k) \right\},
	\end{equation*}
	whose gradient in terms of $\phi$ is derived as
	\begin{equation*}
	\begin{split}
	\nabla_{\phi}\mathcal{J}_{\pi}(\phi) &= \sum\frac{ \left(\alpha\nabla_{a_{k}}\ln \pi_{\phi} + \nabla_{a_{k}}\mathcal{L}_{\theta}(k) \right) \nabla_{\phi}a_{\phi} + \alpha \nabla_{\phi}\ln \pi_{\phi}}{|\mathcal{B}|}.
	\end{split}
	\end{equation*}
	
	For the temperature $\alpha$, it is updated by minimising the following function
	\begin{equation*}
	\mathcal{J}_{\alpha} = \mathbb{E}_{\pi}\left\{-\alpha \ln \pi(a_{k}|\tilde{x}_{k}) - \alpha \mathcal{H} \right\},
	\end{equation*}
	where $\mathcal{H}$ is a target entropy.
	
	For the Lagrangian multiplier $\lambda$, it is learned by maximising
	\begin{equation*}
	\mathcal{J}(\lambda) =\mathbb{E}\left[ \mathcal{L}_{\theta}(k+1)-\mathcal{L}(k) + \beta \Tr(\tilde{x}_{k}\tilde{x}^{\top}_{k})-\delta_{k}\right].
	\end{equation*}
	
	\begin{figure*}[tb]
		\centering
		\includegraphics[width=0.88\textwidth]{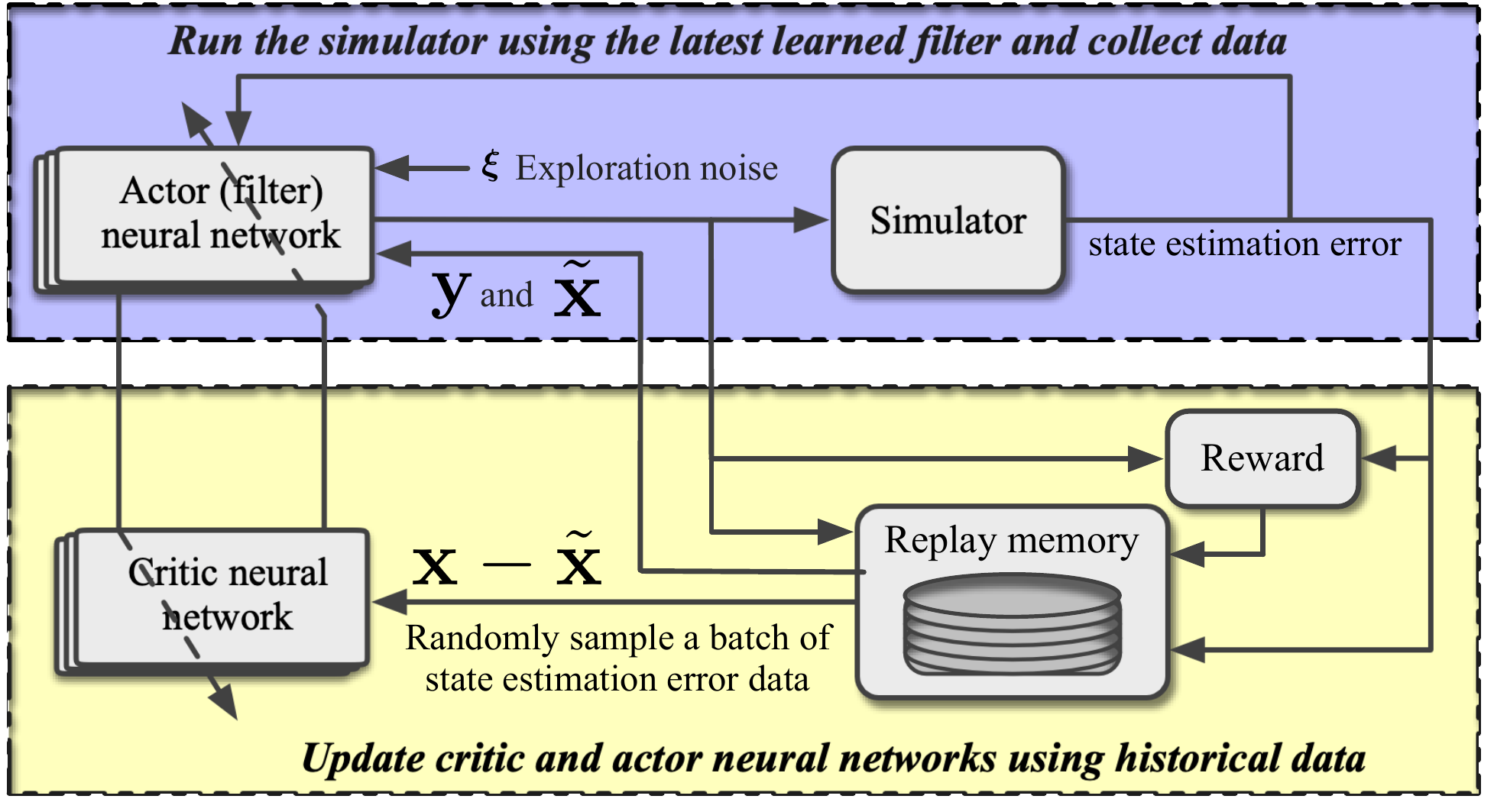}
		\caption{Offline training process of LRLF}
		\label{fig:TrainRL}
	\end{figure*}
	
	Our algorithm is implemented based on SAC algorithm \cite{Haarnoja2018SAC1}, in which $\iota_{\mathcal{L}}, \iota_\pi$, $\iota_\alpha$, and $\iota_{\lambda}$ are the positive learning rates, and $\tau>0$ is a constant scalar.
	The optimal parameters for the DNN in \eqref{eq:MLP_Q} and \eqref{eq:MLP_u} will be learned, and the filter gain policy will be approximated by $\pi_{\phi^{\ast}}$ from which action will be sampled. During inference, the mean value of $\pi_{\phi^{\ast}}$ will be deployed since the policy is often assumed to be Gaussian distributed in SAC \cite{Haarnoja2018SAC1}.

	The inference procedure is illustrated in Fig.~\ref{fig:estimator}. The learned policy is deployed to tune the error $y_{k+1}-g(f(\hat{x}_{k}))$ in the estimator. We sample the measurement output signal $y_{k+1}$ from the real system. Then, the estimator starts to estimate states for the real system.
	
	\begin{figure*}[tb]
		\centering
		\includegraphics[width=0.87\textwidth]{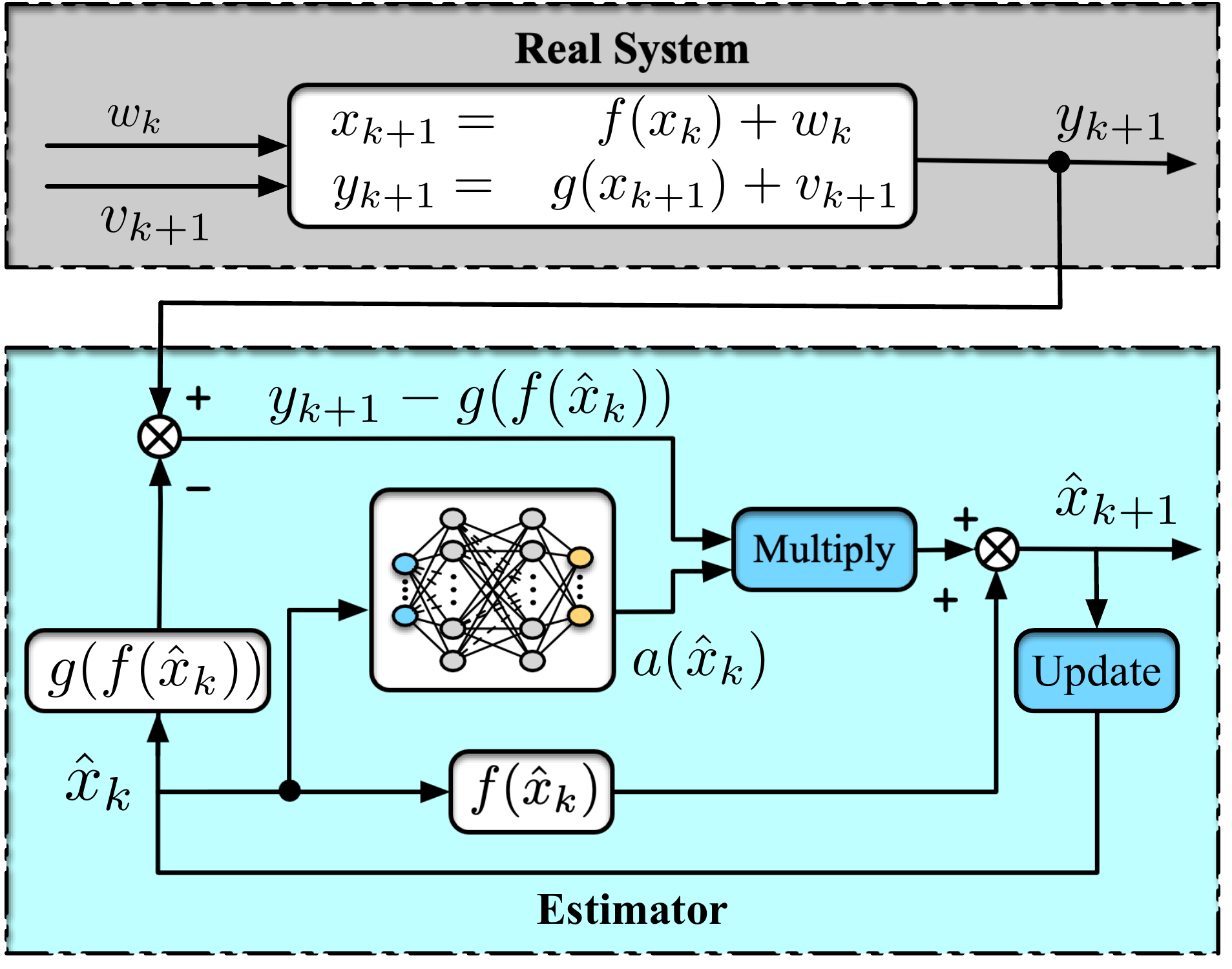}
		\caption{Estimator structure}
		\label{fig:estimator}
		\vspace{-0.1cm}
	\end{figure*}
	
	\begin{algorithm}
		\caption{Lyapunov-Based Reinforcement Learning Filter Algorithm (LRLF)}
		\label{alg1}
		\begin{algorithmic}[1]
			
			\State Set the initial parameters $\theta$ for the Lyapunov function $\mathcal{L}_{\theta}$, $\phi$ for the filtering policy $\pi_{\phi}$, $\lambda$ for the Lagrangian multiplier, $\alpha$ for the temperature parameter, and the replay memory $\mathcal{M}$
			
			\State Set the target parameter $\bar{\theta}$ as $\bar{\theta}\leftarrow \theta$
			
			\While{Training}
			\For {each data collection step}
			\State Choose $a_{k}$ using $\pi_{\theta}(a_{k}|\tilde{x}_{k})$
			\State Run the system (\ref{equ01}) and the filter system (\ref{equ02}) and collect data $\tilde{x}_{k}$
			\State $\mathcal{M} \leftarrow \mathcal{M} \cup \tilde{x}_{k}$
			\EndFor
			\For {each gradient step}
			\State $\theta \leftarrow \theta  - \iota_{\mathcal{L}}\nabla_{\theta}\mathcal{J}_{\mathcal{L}}(\theta)$,
			\State $\phi \leftarrow \phi - \iota_{\pi}\nabla_{\phi}\mathcal{J}_{\pi}(\phi)$
			\State $\alpha \leftarrow \alpha - \iota_{\alpha}\nabla_{\alpha}\mathcal{J}_{\alpha}(\alpha)$
			\State $\lambda \leftarrow \lambda - \iota_{\lambda}\nabla_{\lambda}\mathcal{J}_{\lambda}(\lambda)$
			\State $\phi_{\bar{\theta}} \leftarrow \tau\theta + (1-\tau) \phi_{\bar{\theta}}$,
			\EndFor
			\EndWhile
			\State Output optimal parameters $\theta^{\ast}$, $\phi^{\ast}$, $\lambda^{\ast}$, and $\alpha^{\ast}$
		\end{algorithmic}
	\end{algorithm}
	\vspace{-0.2cm}

	\begin{remark}
		Bayesian nonlinear filtering methods such as the EKF, UKF, PF adjust the estimator gains or sampling weights via online computation. The proposed LRLF is trained offline, and the filter gain is approximated by a DNN which will be deployed directly for online applications. In this paper, like other methods, the training needs full knowledge of the mathematical model, i.e., \eqref{equ01}. It should also be noted that the filter \eqref{equ02} is trained by only using the samples simulated from \eqref{equ01} instead of any other assumptions on the model. In other words, the mathematical model is a simulator, and our training is performed in a model-free manner. In Figs. \ref{fig:TrainRL} and \ref{fig:estimator}, the statistical information of the system's noise such as the covariance is not directly used by the LRLF, which is different from the EKF, UKF and PF where such statistical information is used explicitly.
	\end{remark}

	\subsection{Algorithm convergence analysis}
	
	Next, a lemma is given to show that the policy evaluation can guarantee the action value function to converge. Since the proof is standard, it is omitted here. Readers can refer to \cite{Haarnoja2018SAC1} for more details.
	\begin{lemma}\label{lem1}(Policy evaluation)
		Consider the backup operator $\mathcal{T}^{\pi}$ in (\ref{BacOpe}) and define $\mathcal{L}^{t+1}(k)\define\mathcal{T}^{\pi}\mathcal{L}^{t}(k)$. The sequence $\mathcal{L}^{t+1}(k)$ can converge to a soft value $\mathcal{L}^{\pi}$ of the policy $\pi$ as the iteration $t\rightarrow \infty$.
	\end{lemma}
	
	For policy improvement, a lemma is given to show that the updated policy is better than the last one.
	\begin{lemma}\label{lem2}(Policy improvement)
		Considering the last updated policy $\pi_{\text{old}}$ and the new policy $\pi_{\text{new}}$ to be obtained from (\ref{PolUpd-1}), $\mathcal{L}^{\pi_{\text{new}}(k)} \leq \mathcal{L}^{\pi_{\text{old}}(k)}$ holds, $\forall \tilde{x}_{k}\in \mathcal{S}$ and $\forall a_{k} \in \mathcal{A}$.
	\end{lemma}
	\begin{proof}
		According to (\ref{PolUpd-1}), we can obtain
		\begin{equation*}
		\begin{split}
		\mathbb{E}_{\pi_{\text{new}}}\left[\alpha\ln(\pi_{\text{new}}(a_{k}|\tilde{x}_{k}))+\mathcal{L}_{\pi_{\text{old}}}(k) \right] \leq \\
		\mathbb{E}_{\pi_{\text{old}}}\left[\alpha\ln(\pi_{\text{old}}(a_{k}|\tilde{x}_{k}))+\mathcal{L}_{\pi_{\text{old}}}(k) \right],
		\end{split}
		\end{equation*}
		which implies
		\begin{equation}
		\begin{split}
		\mathbb{E}\left[\mathcal{L}_{\pi_{\text{old}}}+\alpha\ln(\pi_{\text{new}}(a_{k}|\tilde{x}_{k})) \right] \leq V_{\pi_{\text{old}}}(\tilde{x}_{k}).
		\end{split}
		\label{RepeatedE}
		\end{equation}
		
		Then, the following inequality holds
		\begin{equation*}
		\begin{split}
		\mathcal{L}_{\pi_{\text{old}}}(k) &= \mathcal{C}_{k} + \gamma\mathbb{E}_{\tilde{x}_{k+1}}\left[V_{\pi_{\text{old}}}(\tilde{x}_{k}) \right] \\
		& \geq \mathcal{C}_{k} + \gamma\mathbb{E}_{\tilde{x}_{k+1}}\left[\mathbb{E}_{\pi_{\text{new}}}\left[ \right.\right.\\
		& \mathcal{L}_{\pi_{\text{old}}}(k+1) \\
		&\left.\left.+\alpha  \ln(\pi_{\text{new}}(a_{k+1}|\tilde{x}_{k+1})) \right] \right] \\
		&\vdots \\
		& \geq \mathcal{L}_{\pi_{\text{new}}}(k),
		\end{split}
		\end{equation*}
		where \eqref{RepeatedE} is repeatedly used and hence omitted. It completes the proof.
	\end{proof}
	
	Next, a theorem is derived to show that the convergence of Algorithm \ref{alg1} can be guaranteed.
	
	\begin{theorem}\label{theo1}
		Define $\pi_{i} ~(i=1,2,\ldots,\infty)$ as the policy obtained at the $i$-th policy improvement step, starting from any policy $\pi_{0} \in \Pi$, where $\Pi$ is the policy set, then $\pi_{i}$ will converge to an optimal policy $\pi^{\star}$, ensuring $\mathcal{L}_{\pi^{\star}}(k)$ converges to its minimal value  as $k\rightarrow \infty$.
	\end{theorem}
	
	\begin{proof}
		Based on Lemma \ref{lem2}, we know that the policy can achieve a better estimate performance after each policy improvement, that is, $\mathcal{L}_{\pi_{i}}(k) \leq \mathcal{L}_{\pi_{i-1}}(k)$. By repeatedly executing the policy evaluation and policy improvement, a policy $\pi_{i}$ can converge to $\pi^{\star}$ as $i\rightarrow \infty$. Thus, $\mathcal{L}_{\pi^{\star}}(k)$ can converge based on the conclusion in Lemma \ref{lem1}.
	\end{proof}

	\section{Simulation}\label{sec:Sim}
	
	The experiment setup for the LRLF is a three-stage procedure. First, a number of $N$ estimation policies are trained for different initial conditions and noise is sampled from a known distribution. During training, the simulator will generate sample trajectories $\{x^i_{k},y^i_{k}, i=1,\ldots, I, k=1,\ldots,K\}$, where $I$, $K$ are the number of training trajectories and that of the trajectory length respectively. Second, the estimation performance of each DRL-based state estimator \eqref{equ02} will be evaluated by running $M$ Monte Carlo simulations, again for different initial conditions and noise sampled from a known distribution. Finally, the state estimator with the lowest trace of estimate error covariance during inference (unless diverged) will be deployed online and used for comparison with other nonlinear filtering algorithms.

	We first consider a free-pendulum tracking example widely used as a benchmark for nonlinear state estimation \cite{sarkka2013bayesian,morimoto2007reinforcement,deisenroth2011robust}. The pendulum has unity mass of $1$ kg and length of $1$ m. The discrete-time dynamics of the pendulum is given as follows:
	\begin{equation}\label{eq1:pendulum}
	\begin{bmatrix}
	x_{1,k+1}\\x_{2,k+1}
	\end{bmatrix}=\begin{bmatrix}
	x_{1,k}+x_{2,k}\delta t\\
	x_{2,k}-g\sin(x_{1,k})\delta t
	\end{bmatrix}+w_{k}
	\end{equation}
	where $x_{1,k}=\theta_{k}$, $x_{2,k}=\omega_{k}$ are the angle and angle velocity of the pendulum at time instant $k$, respectively, $\delta t$ is the sampling time and set as 0.1 second in the simulation. The process noise $w_{k}$ is Gaussian distributed as $$w_{k} \sim \mathcal{N}\bigg(\begin{bmatrix}
	0\\0
	\end{bmatrix},\begin{bmatrix}
	\frac{1}{3}(\delta t)^3q_1 & \frac{1}{2}(\delta t)^2q_1  \\\frac{1}{2}(\delta t)^2q_1 & \delta tq_1 \end{bmatrix}\bigg), q_1=0.01$$
	The measurement equation is given as:
	\begin{equation}\label{eq2:pendulum}
	y_{k}=\sin(x_{1,k})+v_{k}
	\end{equation}
	where the measurement noise is also Gaussian distributed with $v_{k} \sim \mathcal{N}(0,0.01)$. Since the scalar measurement solely depends on the angle ($x_1$), the estimate of the latent state $x_2$ has to be reconstructed using the cross-correlation information
	between the angle and the angular velocity in the dynamics \eqref{eq1:pendulum}.

	To test if the estimate error of LRLF converges regardless of the initial state and estimate, for each trajectory in the training, the initial state were sampled from uniform distribution:
	\begin{equation}
	\begin{aligned}
	\theta(0)&\sim \mathcal{U} [-0.5\pi, 0.5\pi],\\
	\omega(0)&\sim \mathcal{U} [-0.5\pi, 0.5\pi],
	\label{eq:initial_state}
	\end{aligned}
	\end{equation}
	and the initial estimate was sampled from uniform distribution:
	\begin{equation}
	\begin{aligned}
	\hat{\theta}(0)&\sim \mathcal{U} [-0.25\pi+\theta(0), 0.25\pi+\theta(0)],\\
	\hat{\theta}(0)&\sim \mathcal{U} [-0.25\pi+\omega(0), 0.25\pi+\omega(0)].
	\label{eq:initial_estimate}
	\end{aligned}
	\end{equation}
	
	We compared LRLF with three other classic nonlinear Bayesian estimation algorithms, the EKF, UKF, and PF ($10^3$ and $10^4$ particles respectively) \cite{sarkka2013bayesian}. The same initial state and state estimate in \eqref{eq:initial_state} and \eqref{eq:initial_estimate} were used for all estimation algorithms.
	
	In the training, each trajectory has $K=100$ data points ($10$ seconds simulation of \eqref{eq1:pendulum}). We trained $N=10$ policy networks and  evaluated each network for $M=500$ Monte Carlo simulations. The training details are give as follows:
	\begin{table}[htb]
		\caption{Hyperparameters of LRLF}\label{table:hyperparameters}
		\vspace{-0.2cm}
		\begin{center}
			\begin{tabular}{l|c c}
				Hyperparameters&Pendulum&Vehicle\\\hline
				Time horizon $K$&100&100\\
				Minibatch size& 256& 256\\
				Actor learning rate & 1e-4& 1e-4\\
				Critic learning rate & 3e-4& 3e-4\\
				Lyapunov learning rate & 3e-4& 3e-4\\
				Target entropy& NaN&NaN\\
				Soft replacement($\tau$) &0.005&0.005\\
				Discount($\gamma$)  &0.995&0.995 \\
				$\alpha_3$&0.1& 0.1 \\
				Structure of $a_\phi$ & (32,16)&(32,16)\\
				Structure of $L_\theta$ & (64,32) &(64,32)\\
			\end{tabular}
		\end{center}
		\vspace{-0.2cm}
	\end{table}
	
	For the LRLF, there are two networks: the policy network and the Lyapunov critic network. We use a fully-connected MLP with one hidden layer for the policy network, outputting the mean and SD of a Gaussian distribution. As mentioned in \autoref{sec:algorithm}, it should be noted that the output of the Lyapunov critic network is a square term, which is always non-negative. More specifically, we use a fully-connected MLP with one hidden layer and one output layer with different units as in \autoref{table:hyperparameters}, outputting the feature vector $Q_{\theta}$ (see Fig. \ref{fig:ACNN}). The Lyapunov critic function is obtained by $L_c(s,a)=Q^{\top}_{\theta}(s,a)Q_{\theta}(s,a)$. All the hidden layers use ReLu activation function, and we adopt the same invertible squashing function technique as~\cite{Haarnoja2018SAC1} to the output layer of the policy network.
	
	The proposed LRLF was evaluated for the following aspects:
	\begin{enumerate}
		\item Algorithm convergence: does the proposed training algorithm converge with random initial states and estimate initialisation?
		
		\item Estimate error convergence: does the estimate error variance converge compared with other state estimation algorithms?
		
		\item Performance comparison: how does the proposed LRLF perform compared with other state estimation algorithms under various initial state conditions?
		
		\item Robustness to uncertainty: how do the trained estimator perform during inference when faced with uncertainties unseen during training, such as noise with covariance shift and randomly-occurring missing measurement?
		
	\end{enumerate}

	\subsection{Algorithm convergence}
	
	We will validate the convergence guarantee by checking the values of Lagrange multipliers.  When the Lyapunov constraint in \eqref{PolUpd-2} is satisfied, the parameter $\lambda$ should continuously decrease to zero.  In Fig. \ref{fig:lambda}, the value of $\lambda$ during training is demonstrated. In all training trials of the ten policy networks,  $\lambda$ converges to zero eventually, which implies the state estimate's convergence guarantee.

	\begin{figure}[tb]
		\centering
		\includegraphics[width=0.38\textwidth]{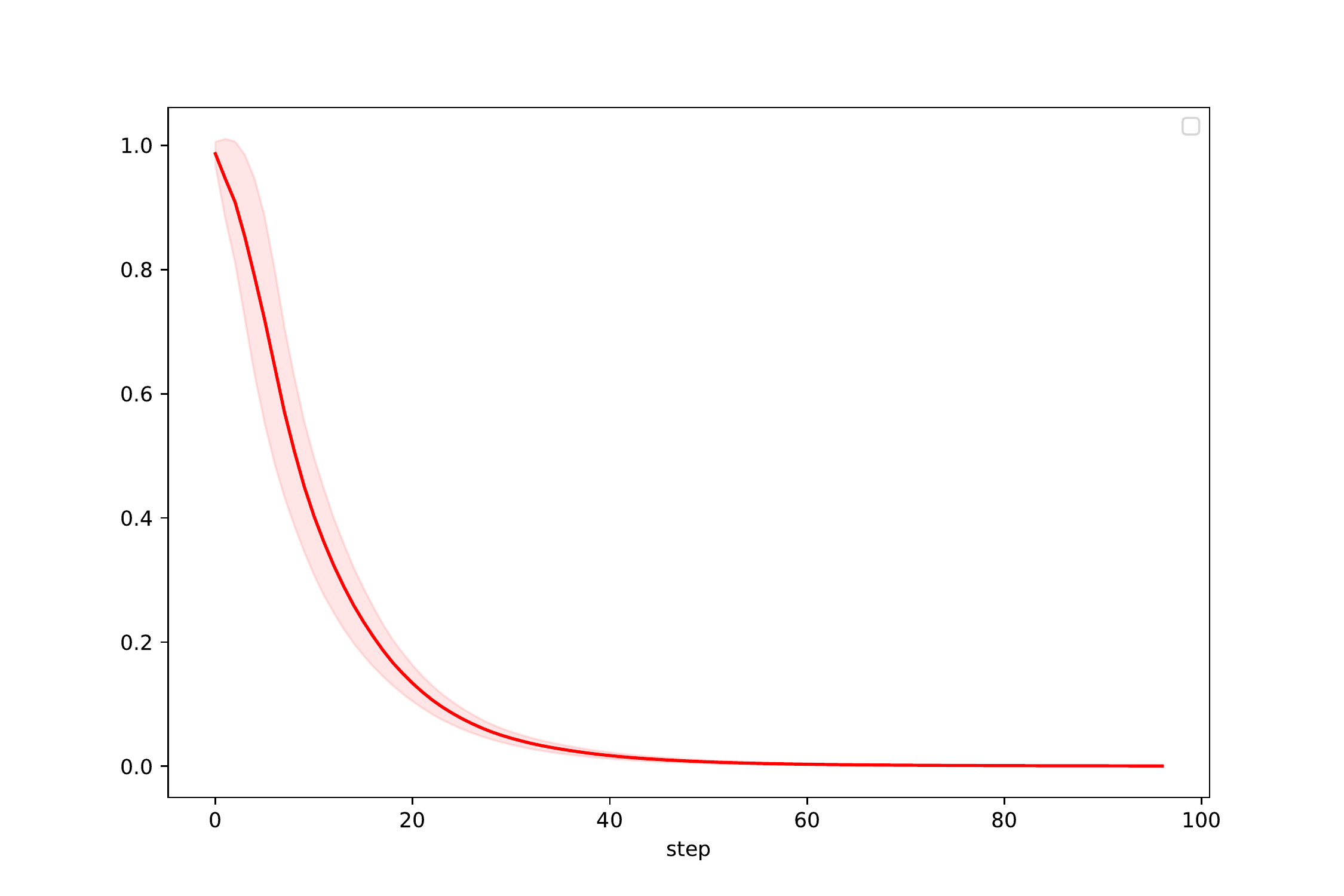}
		\caption{The value of Lagrange multiplier $\lambda$ during the training of LAC policies. The Y-axis indicates the value of $\lambda$ and the X-axis indicates the total time steps. The shadowed region shows the 1-SD confidence interval over 10 randomly training policies.}
		\label{fig:lambda}
		\vspace{-0.3cm}
	\end{figure}

	\begin{figure*}
		\centering
		\subfloat[EKF]{
			\begin{minipage}{.3\textwidth}
				\centering
				\includegraphics[width=0.95\textwidth]{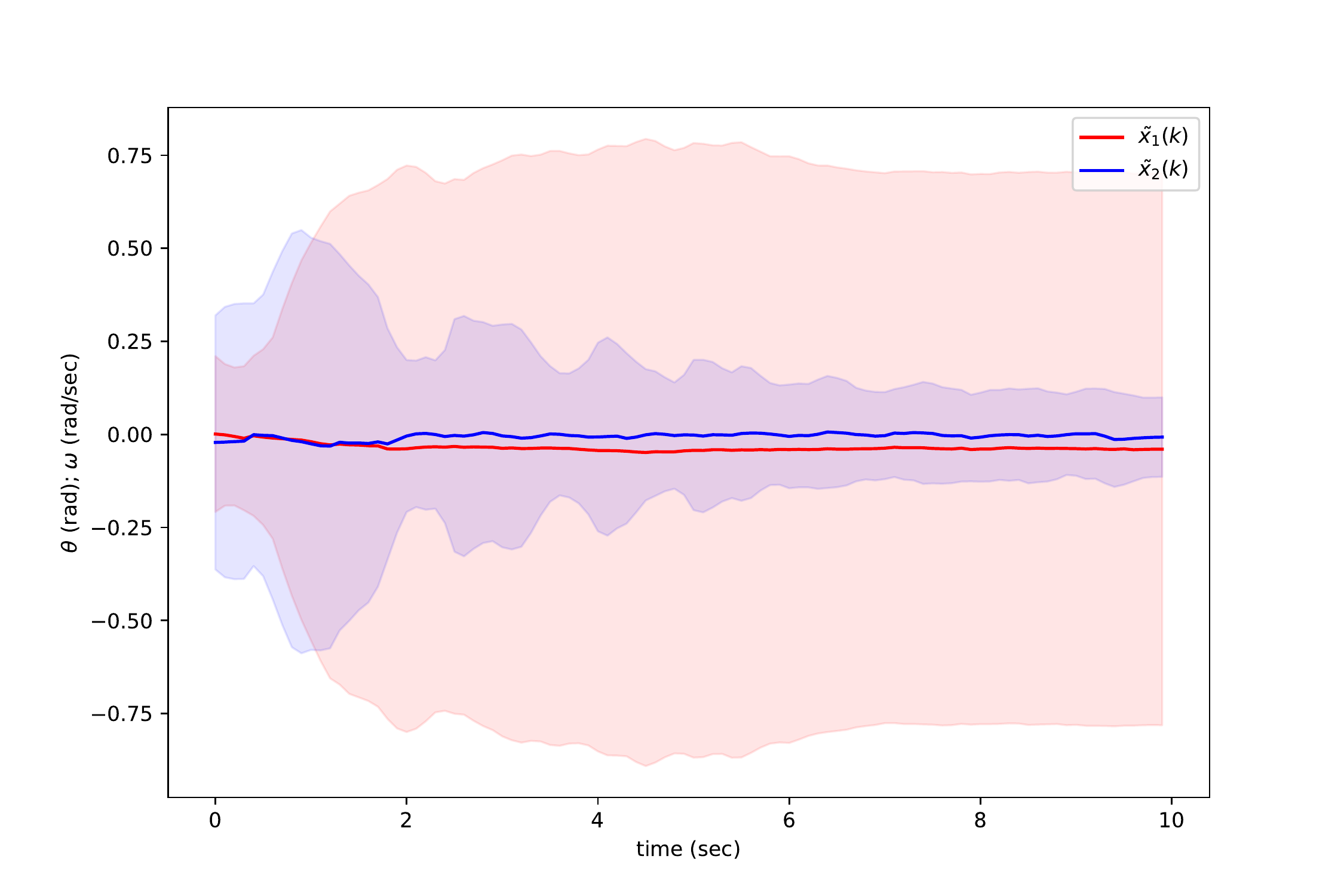}
		\end{minipage}}%
		\subfloat[EKF with noise variance $R=0.1$]{
			\begin{minipage}{.3\textwidth}
				\centering
				\includegraphics[width=0.95\textwidth]{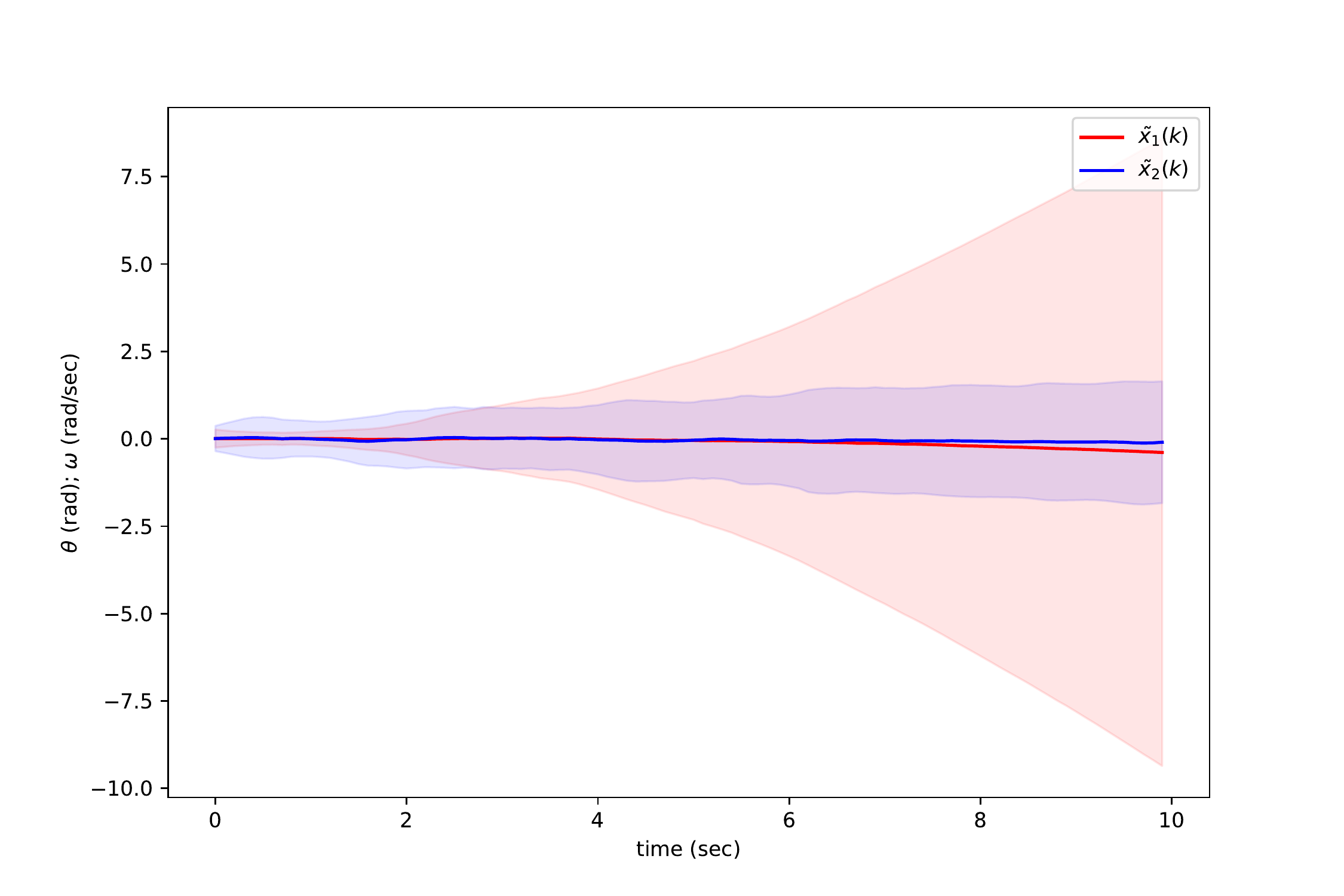}
		\end{minipage}}%
		\subfloat[EKF with measurement missing]{
			\begin{minipage}{.3\textwidth}
				\centering
				\includegraphics[width=0.95\textwidth]{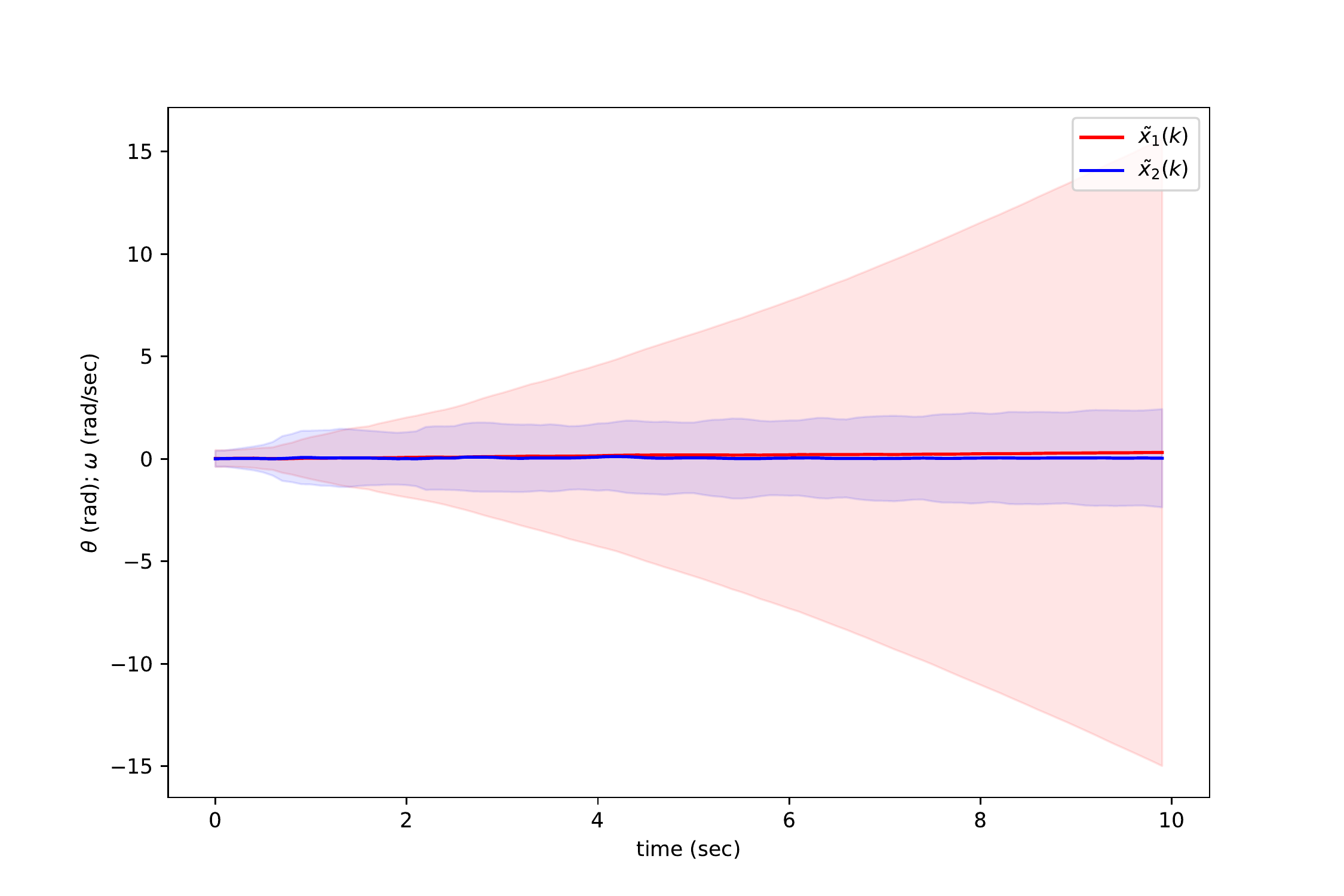}
		\end{minipage}}%
		
		\subfloat[UKF]{
			\begin{minipage}{.3\textwidth}
				\centering
				\includegraphics[width=0.95\textwidth]{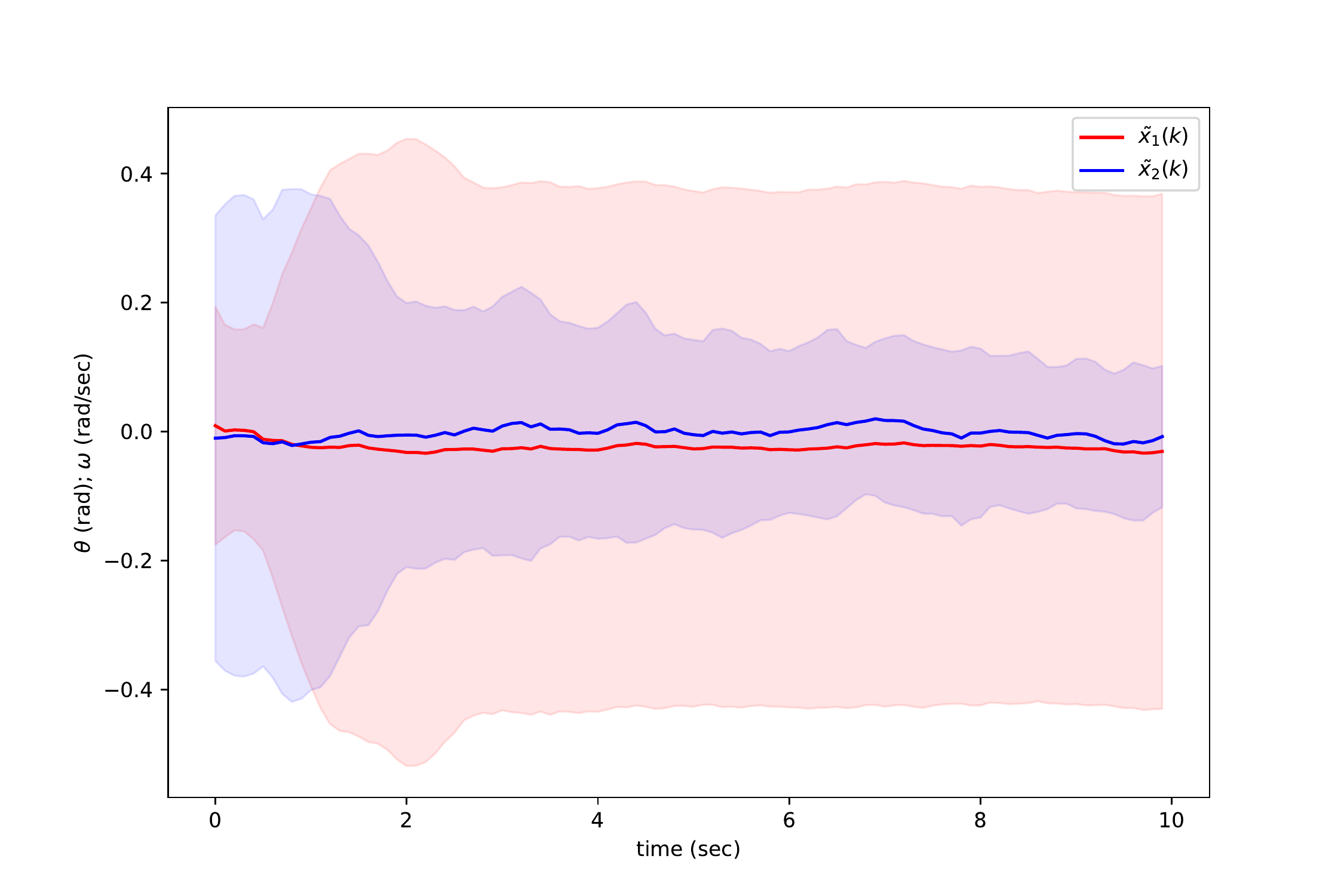}
		\end{minipage}}
		\subfloat[UKF with noise variance $R=0.1$]{
			\begin{minipage}{.3\textwidth}
				\centering
				\includegraphics[width=0.95\textwidth]{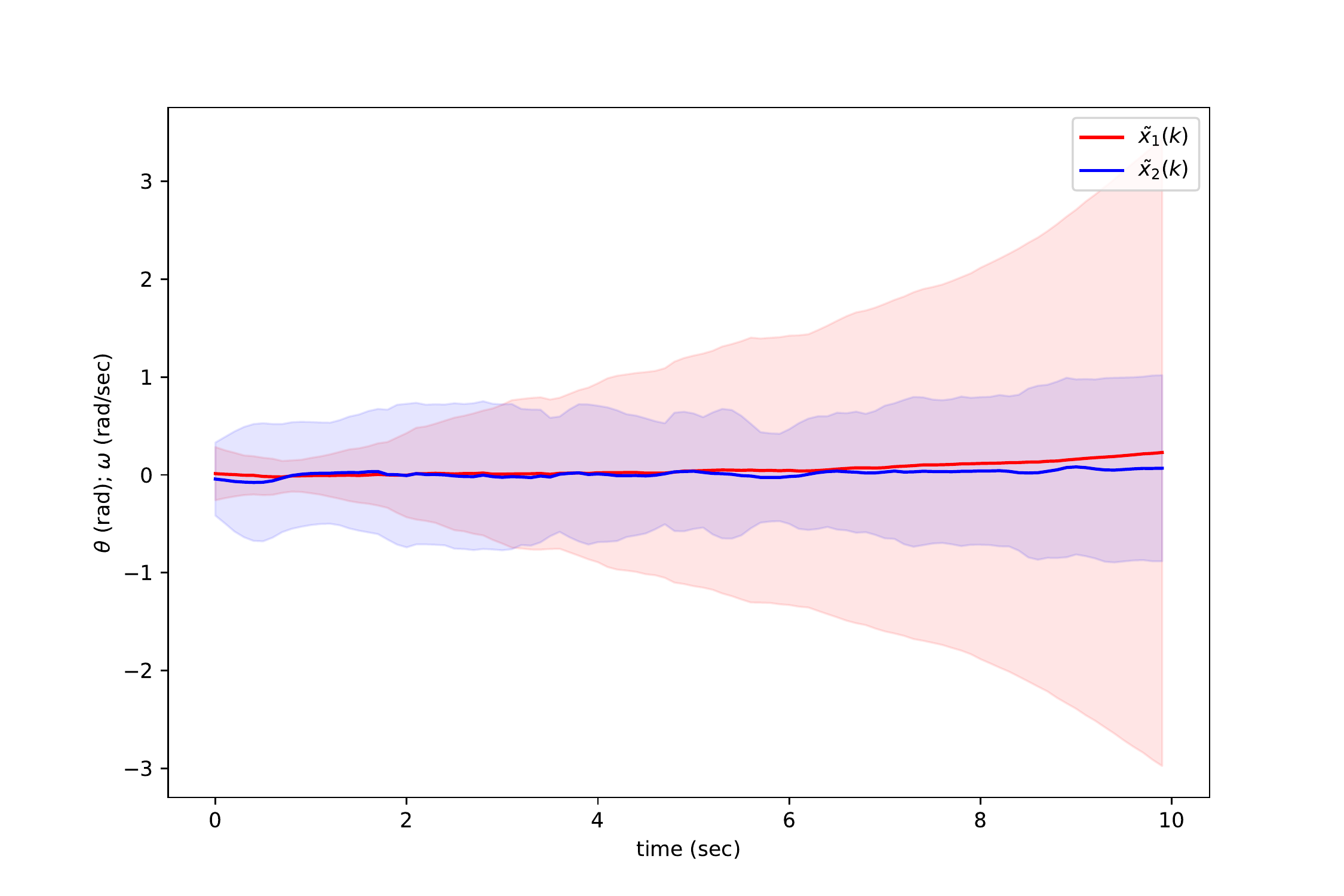}
		\end{minipage}}
		\subfloat[UKF with measurement missing]{
			\begin{minipage}{.3\textwidth}
				\centering
				\includegraphics[width=0.95\textwidth]{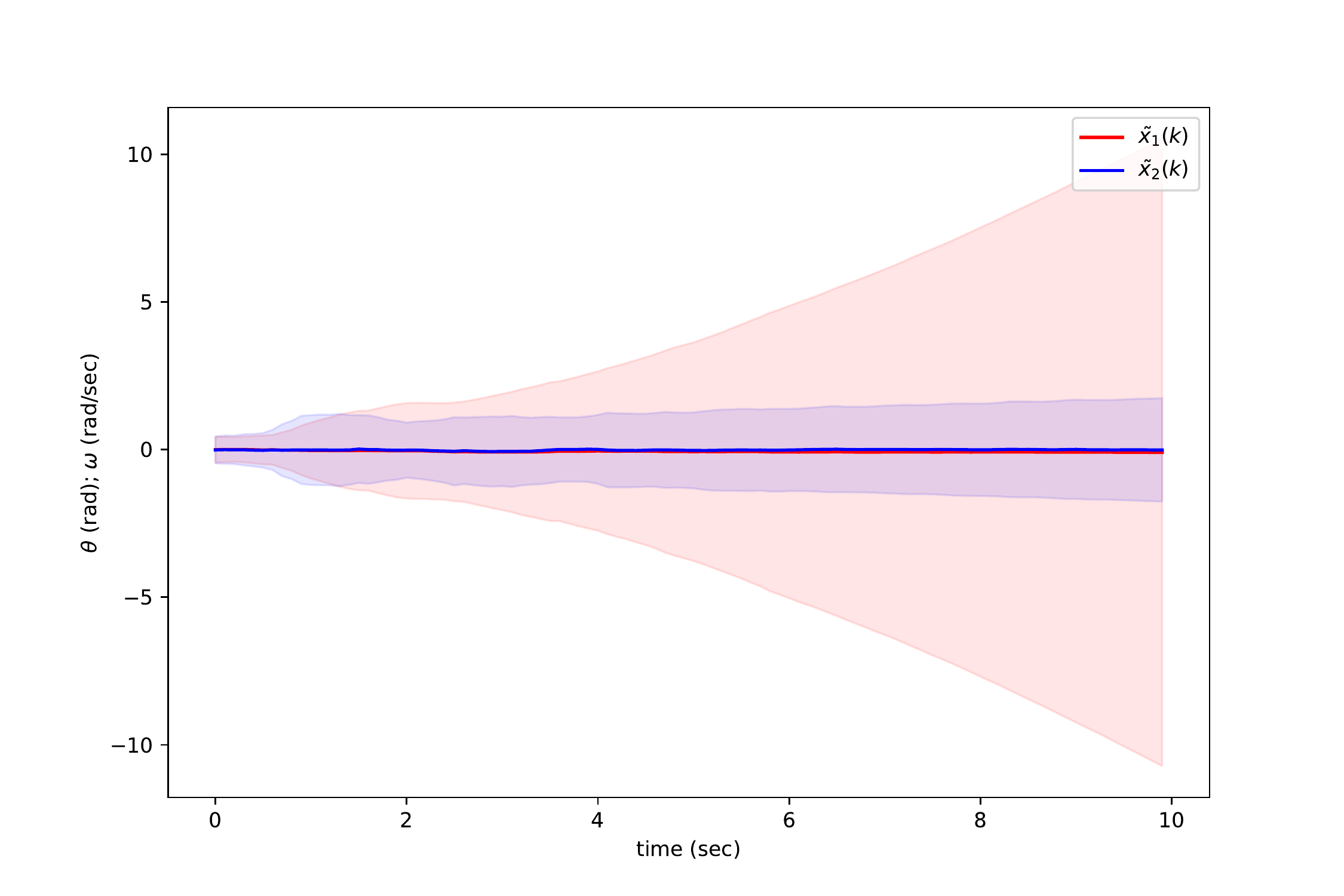}
		\end{minipage}}
		
		\subfloat[PF ($10^3$ particles)]{
			\begin{minipage}{.3\textwidth}
				\centering
				\includegraphics[width=0.95\textwidth]{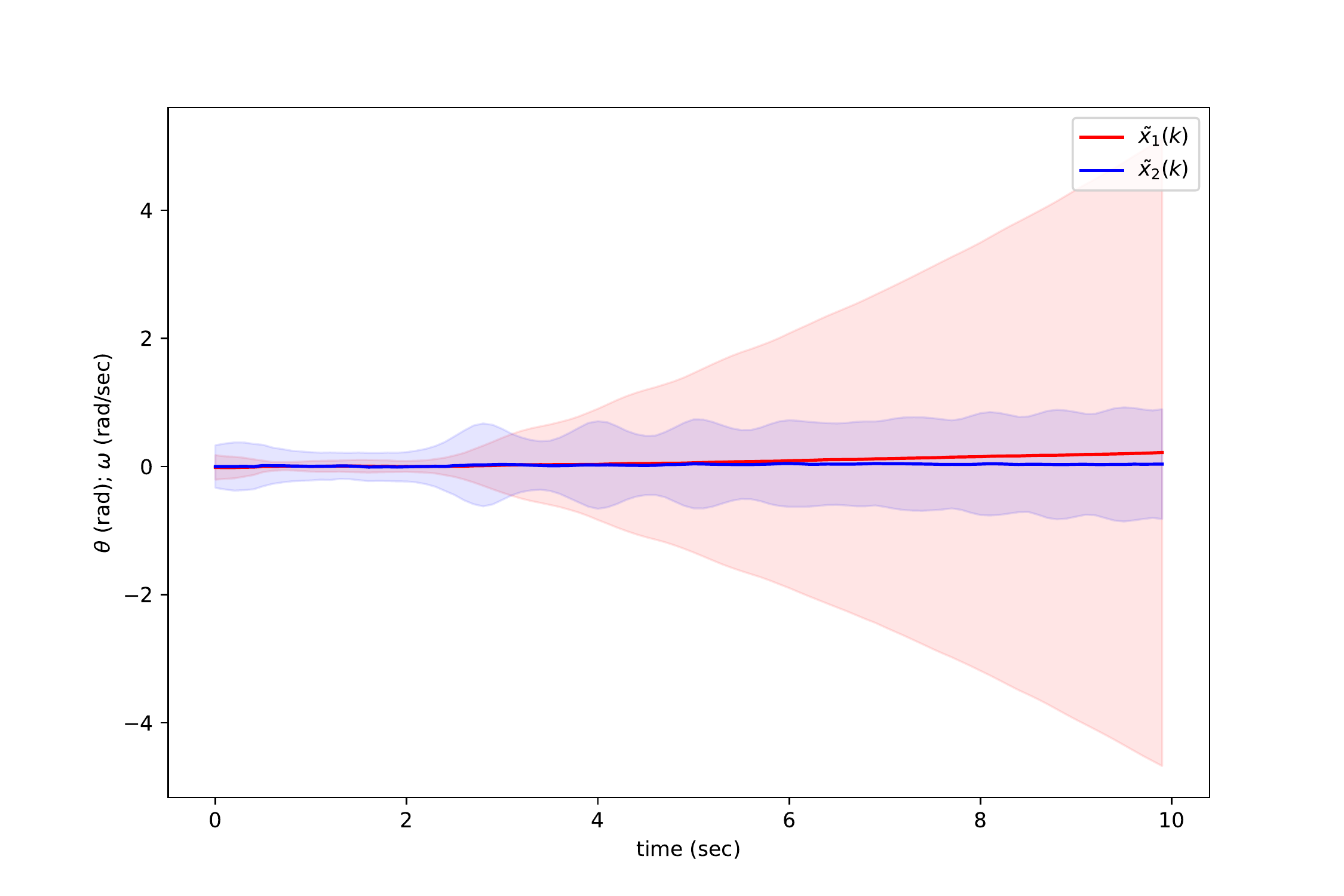}
		\end{minipage}}
		\subfloat[PF with noise variance $R=0.1$]{
			\begin{minipage}{.3\textwidth}
				\centering
				\includegraphics[width=0.95\textwidth]{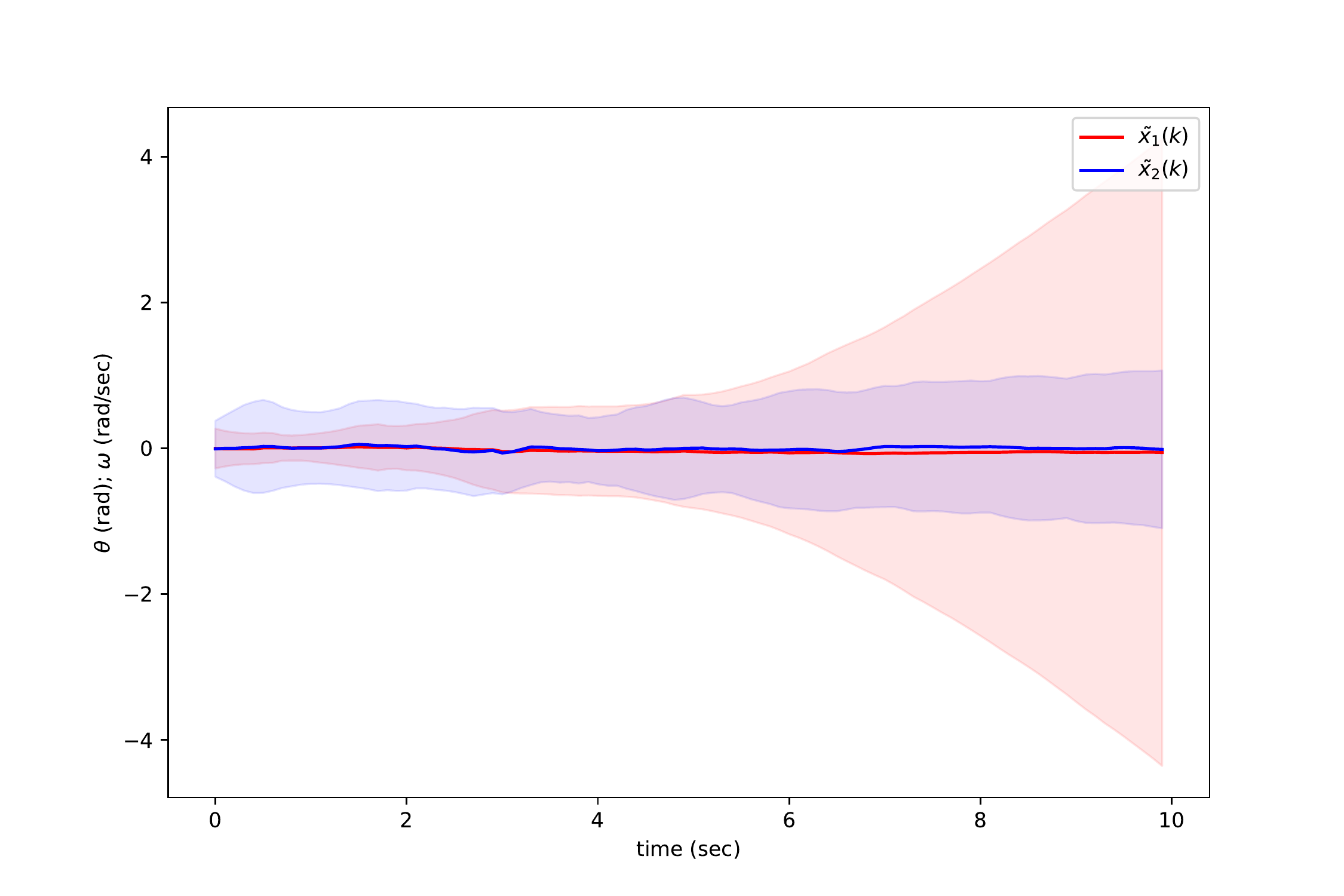}
		\end{minipage}}
		\subfloat[PF with measurement missing]{
			\begin{minipage}{.3\textwidth}
				\centering
				\includegraphics[width=0.95\textwidth]{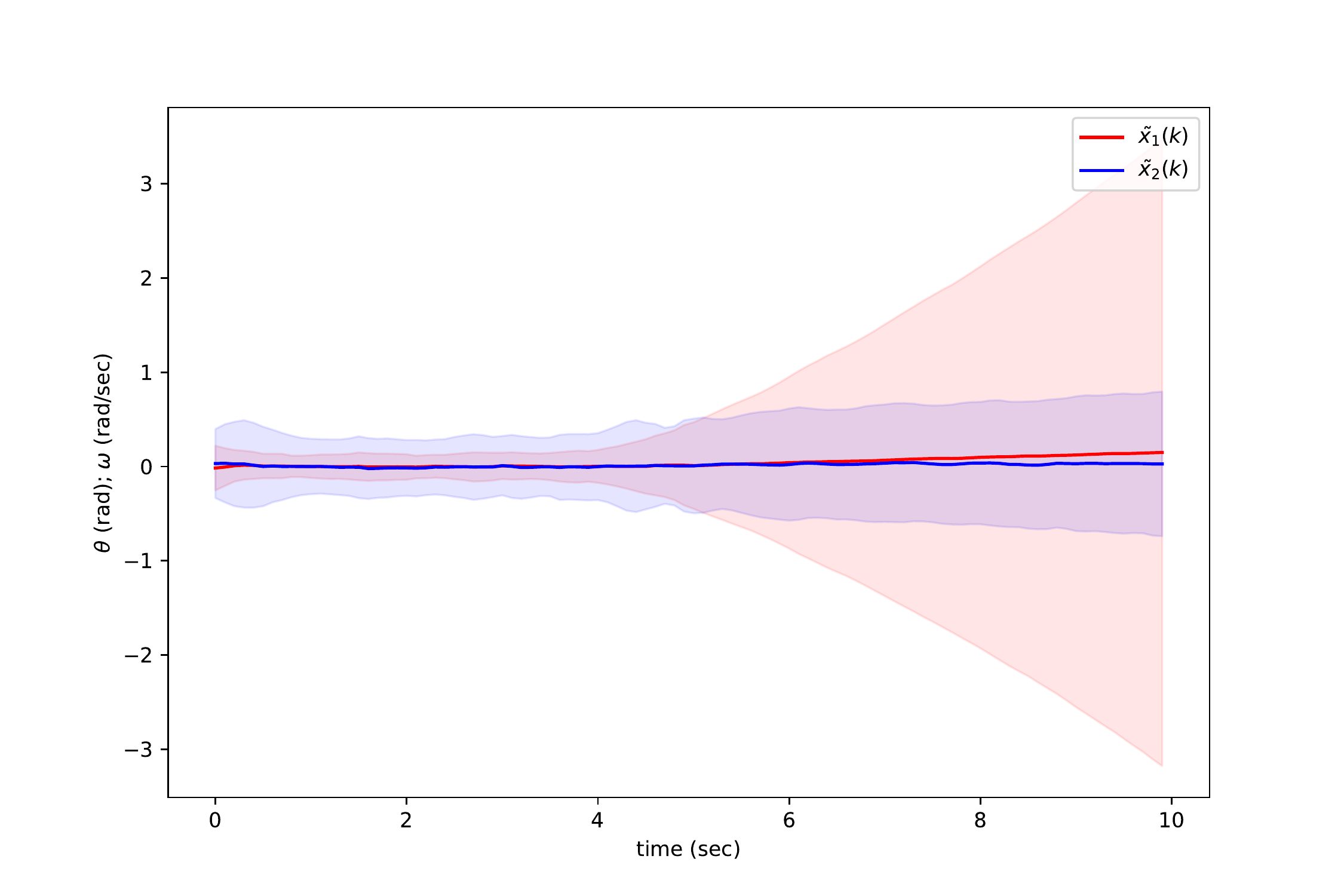}
		\end{minipage}}
		
		\subfloat[PF ($10^4$ particles)]{
			\begin{minipage}{.3\textwidth}
				\centering
				\includegraphics[width=0.95\textwidth]{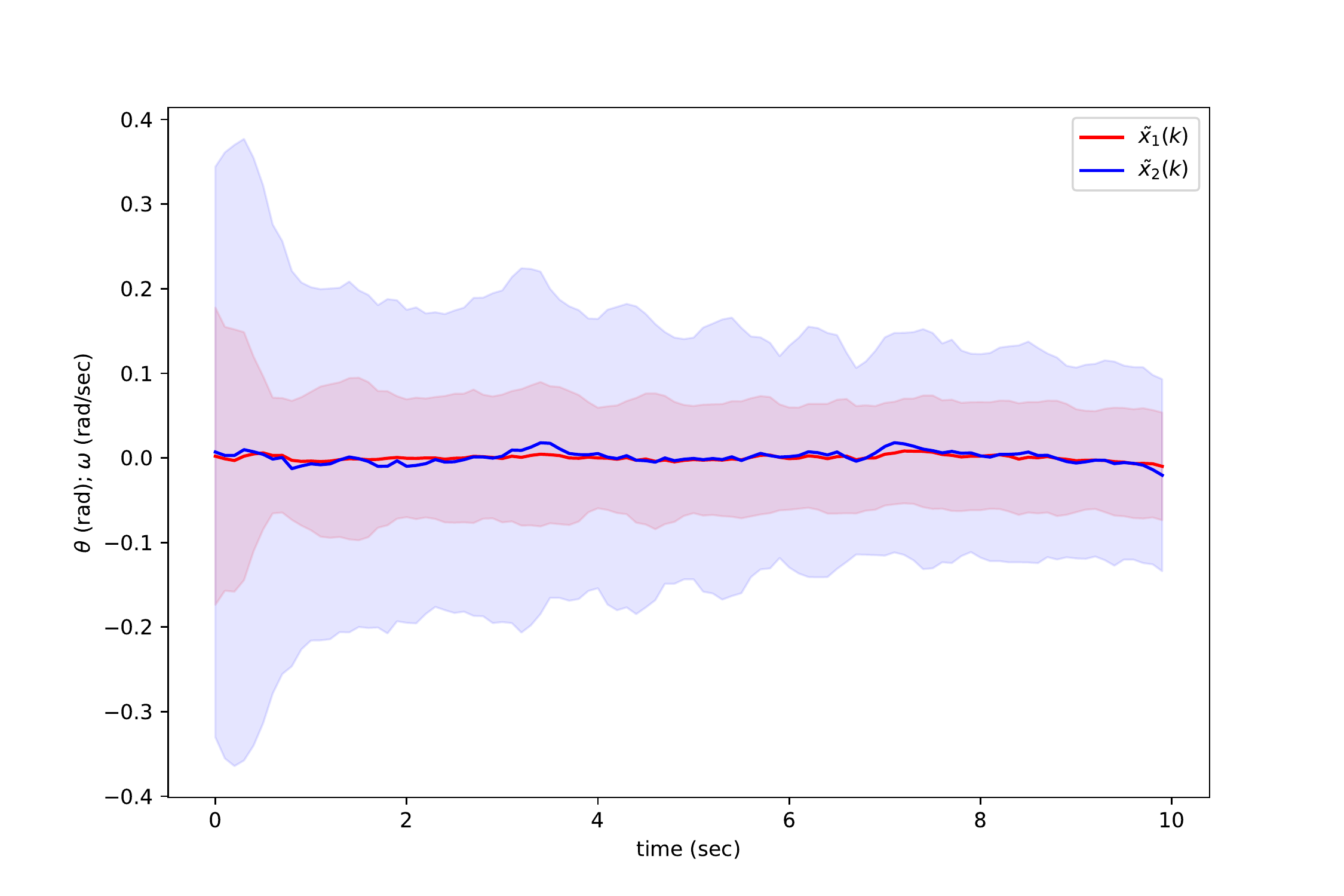}
		\end{minipage}}
		\subfloat[PF with noise variance $R=0.1$]{
			\begin{minipage}{.3\textwidth}
				\centering
				\includegraphics[width=0.95\textwidth]{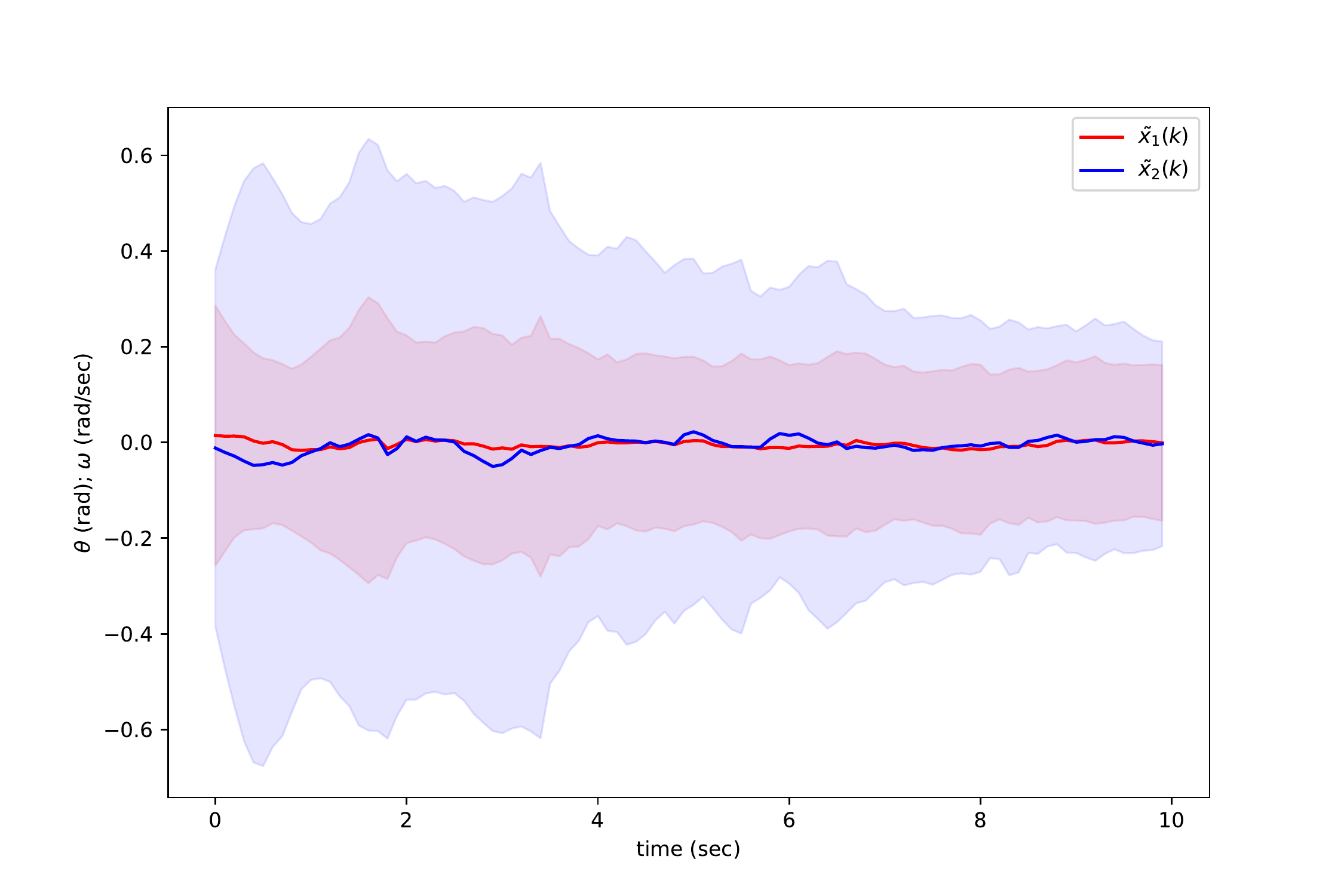}
		\end{minipage}}	
		\subfloat[PF with measurement missing]{
			\begin{minipage}{.3\textwidth}
				\centering
				\includegraphics[width=0.95\textwidth]{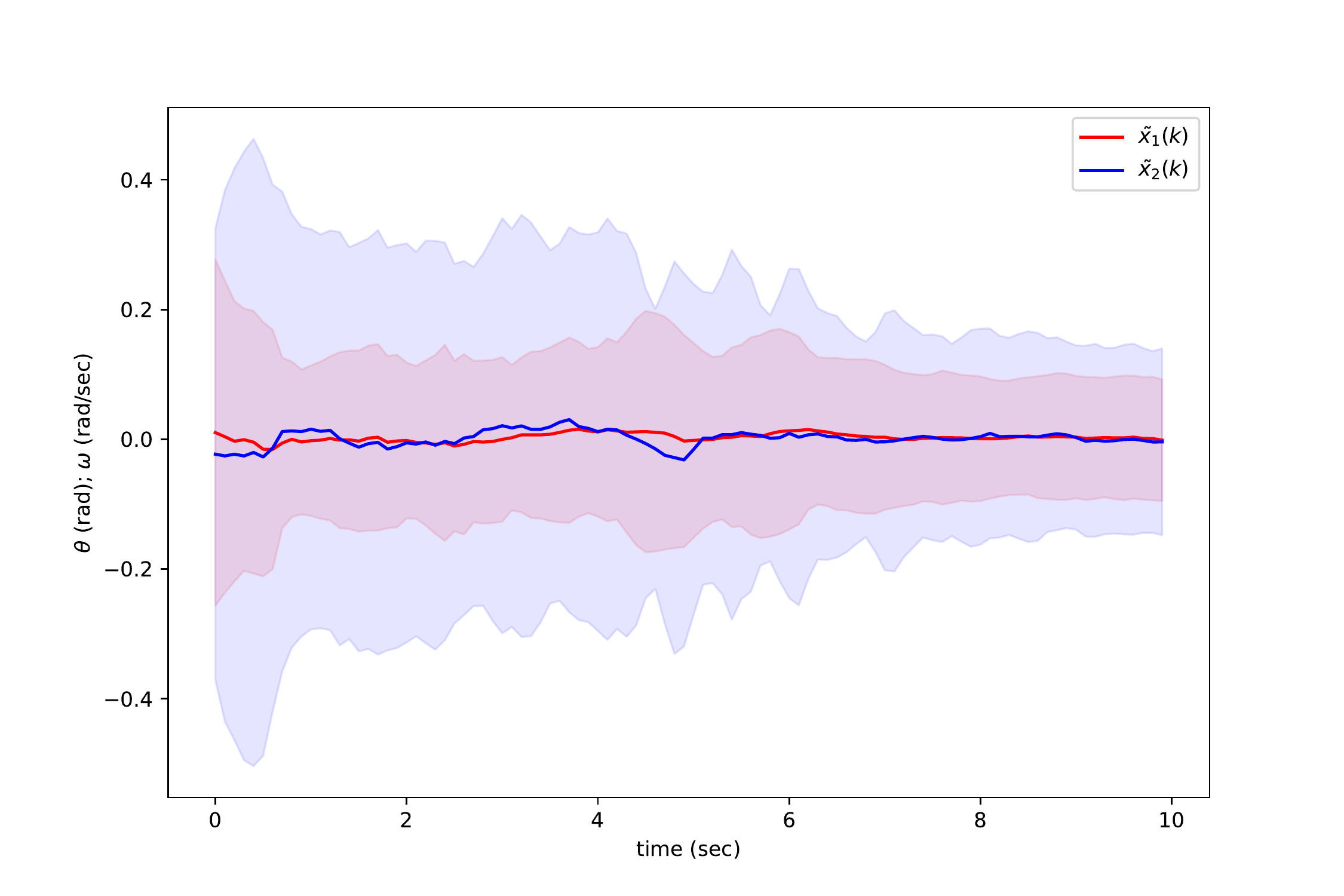}
		\end{minipage}}	
		
		\subfloat[LRLF]{	
			\begin{minipage}{.3\textwidth}
				\centering
				\includegraphics[width=0.95\textwidth]{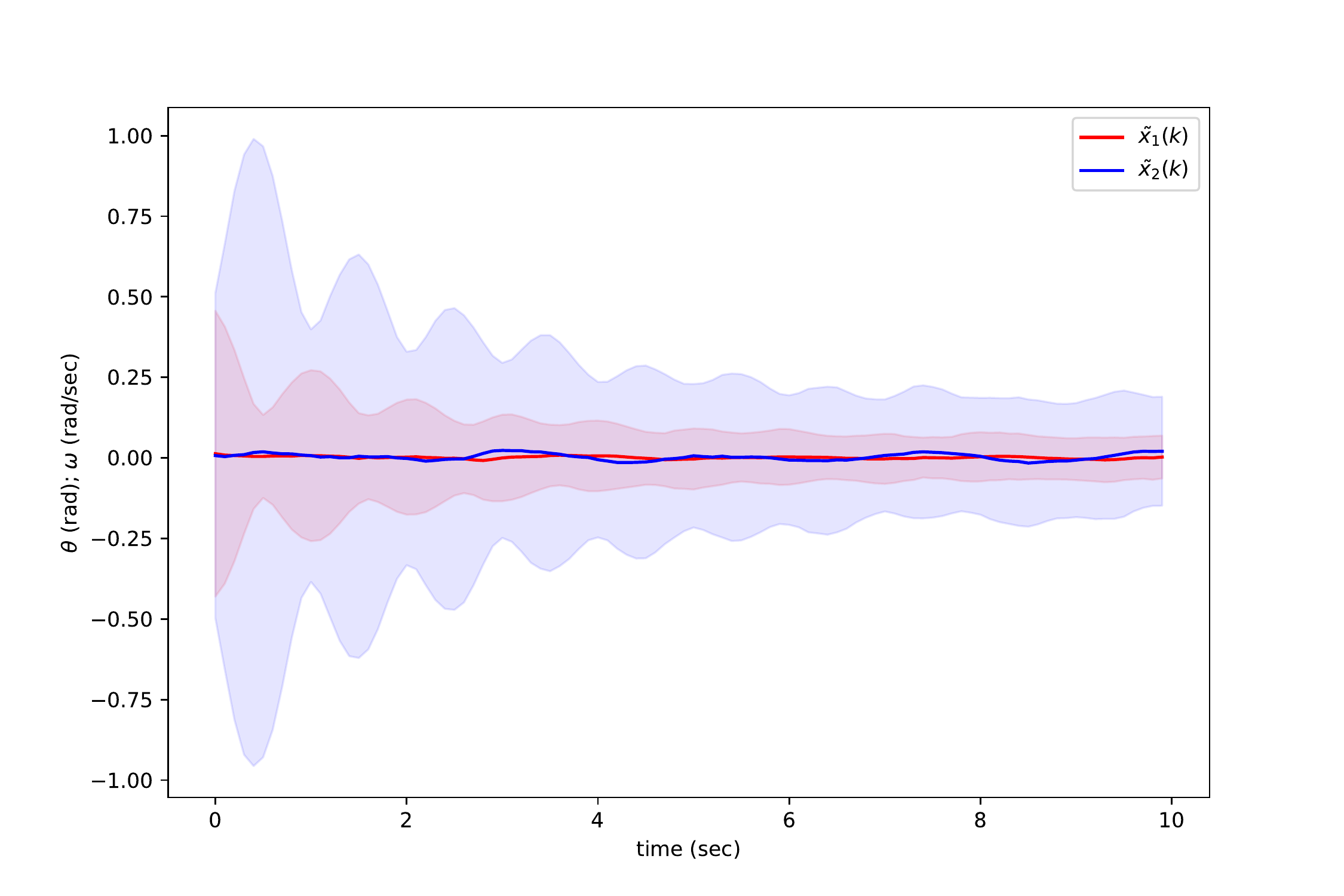}
		\end{minipage}}
		\subfloat[LRLF with noise variance $R=0.1$]{
			\begin{minipage}{.3\textwidth}
				\centering
				\includegraphics[width=0.95\textwidth]{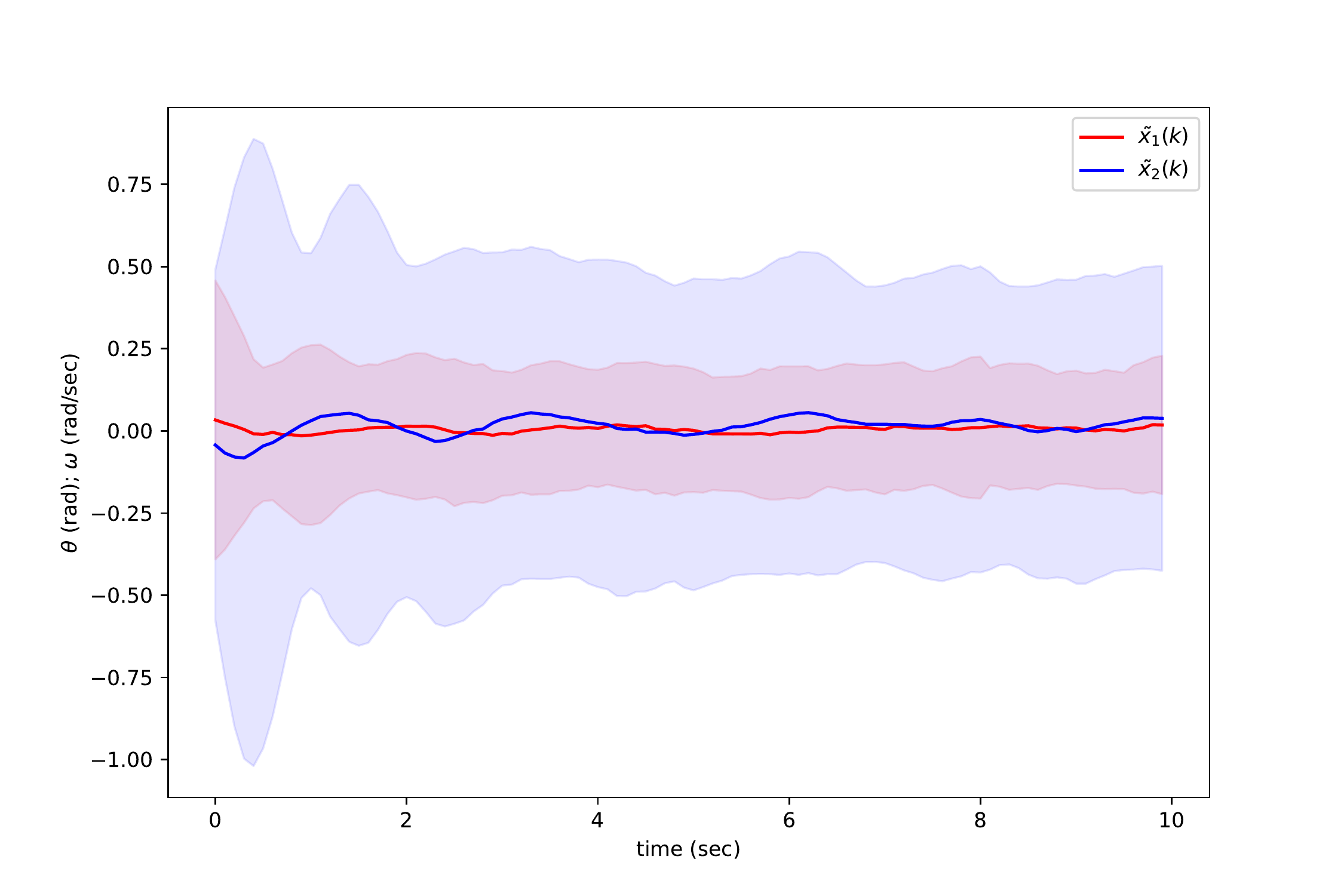}
		\end{minipage}}
		\subfloat[LRLF with measurement missing]{
			\begin{minipage}{.3\textwidth}
				\centering
				\includegraphics[width=0.95\textwidth]{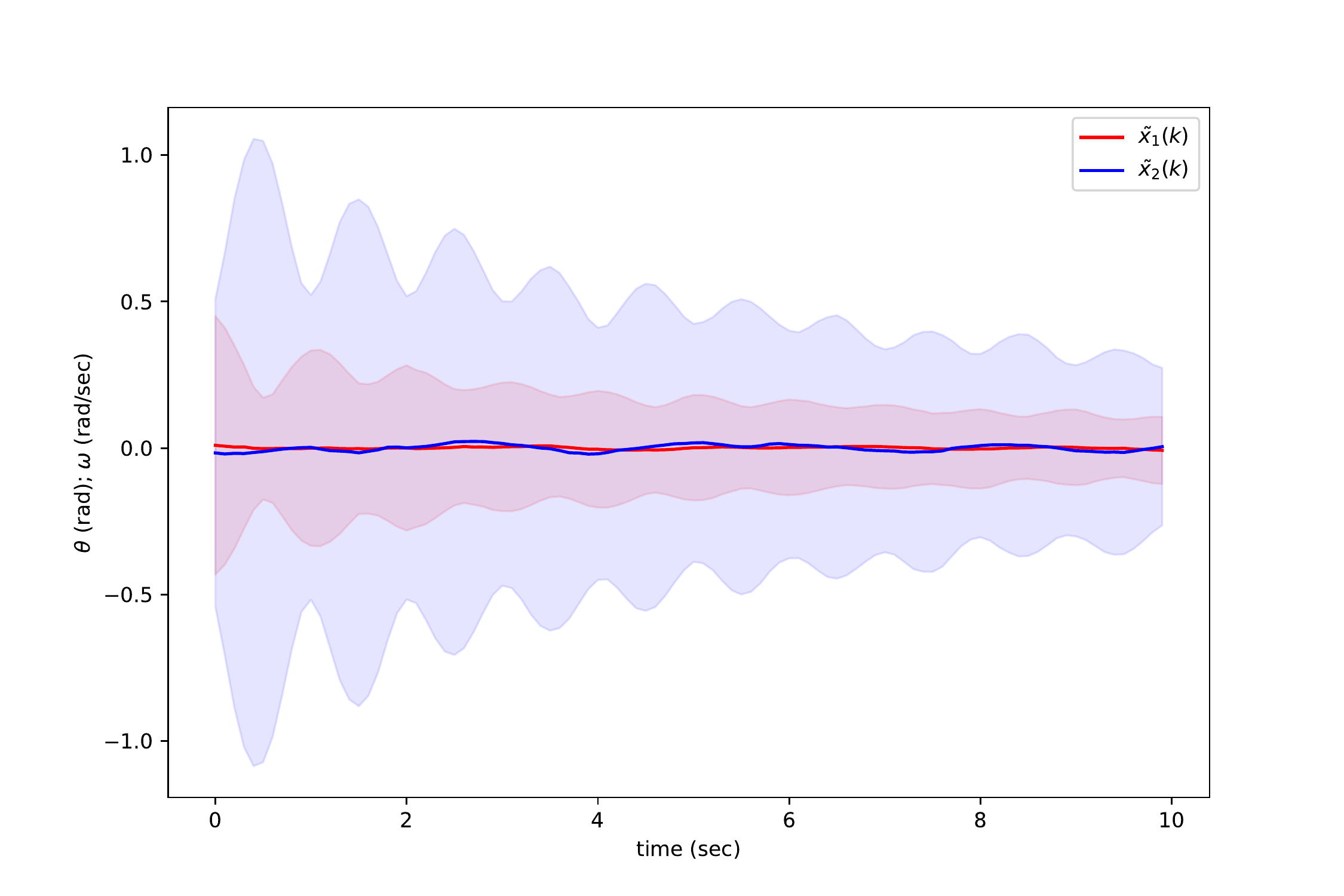}
		\end{minipage}}
		\caption{{\small State estimate error of different methods under different conditions. The left column corresponds to those with simulation setup with the measurement noise variance $R=0.01$; the middle column corresponds to those with the increased measurement noise variance $R=0.1$; the right column corresponds to those with measurement missing happening with the probability $0.5$. Solid line indicates the average estimate error $\tilde{x}_{1,k}$  (red) and $\tilde{x}_{2,k}$  (blue) and shadowed region for the 1-SD confidence interval. All results are obtained from 500 Monte Carlo runs.}}
		\label{fig:estimate_error}
	\end{figure*}

	\subsection{Estimation performance}
	\begin{figure}
		\centering
		\includegraphics[width=0.45\textwidth]{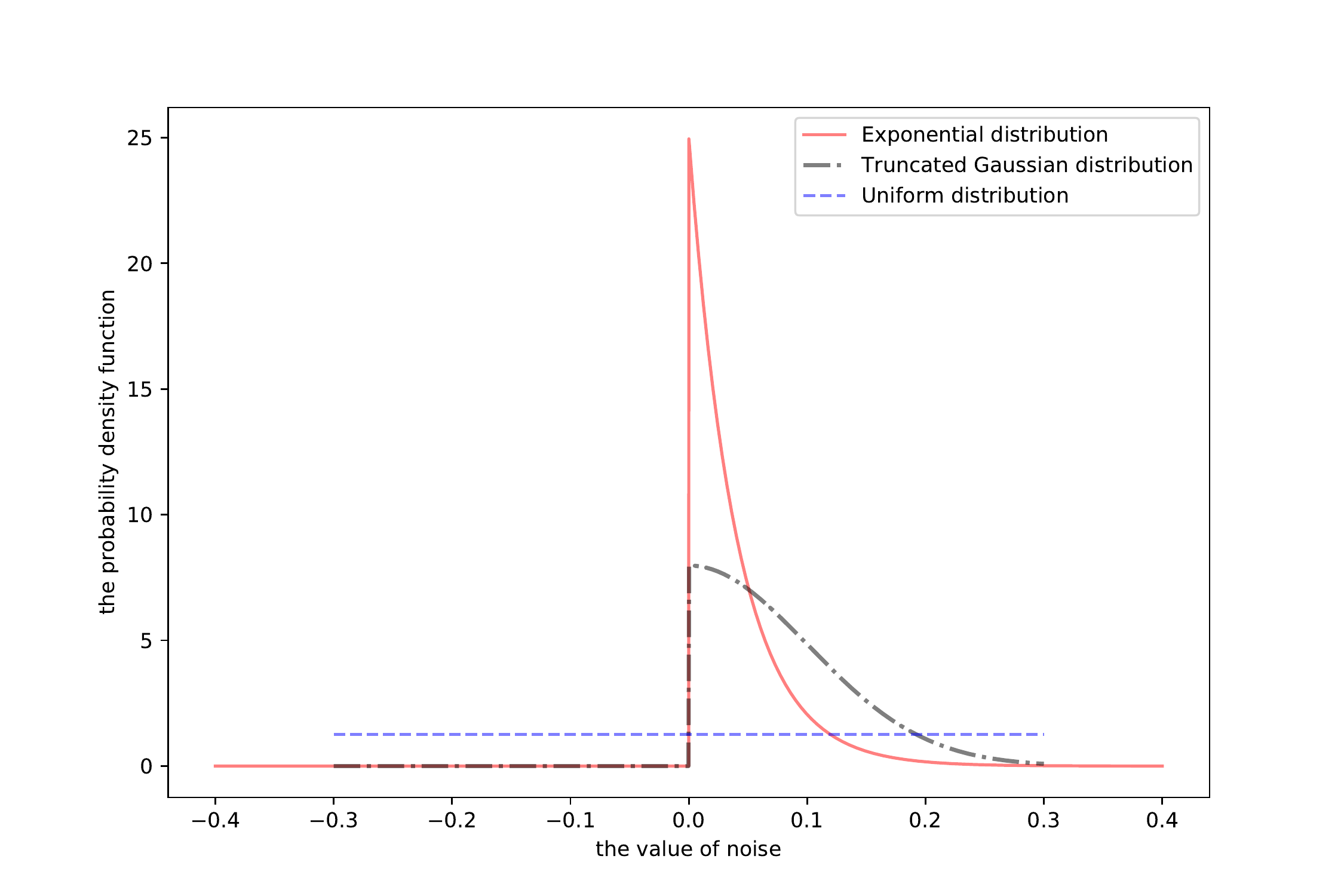}
		\caption{The PDFs of the three kinds of non-Gaussian noise, e.g., the truncated Gaussian, the exponential and the uniform distributions that we tested with the LRLF.}
		\label{fig:noise}
		\vspace{-0.5cm}
	\end{figure}

	\begin{figure*}[tb]
		\centering
		\subfloat[Truncated Gaussian distribution]{	
			\begin{minipage}{.33\textwidth}
				\centering
				\includegraphics[width=0.95\textwidth]{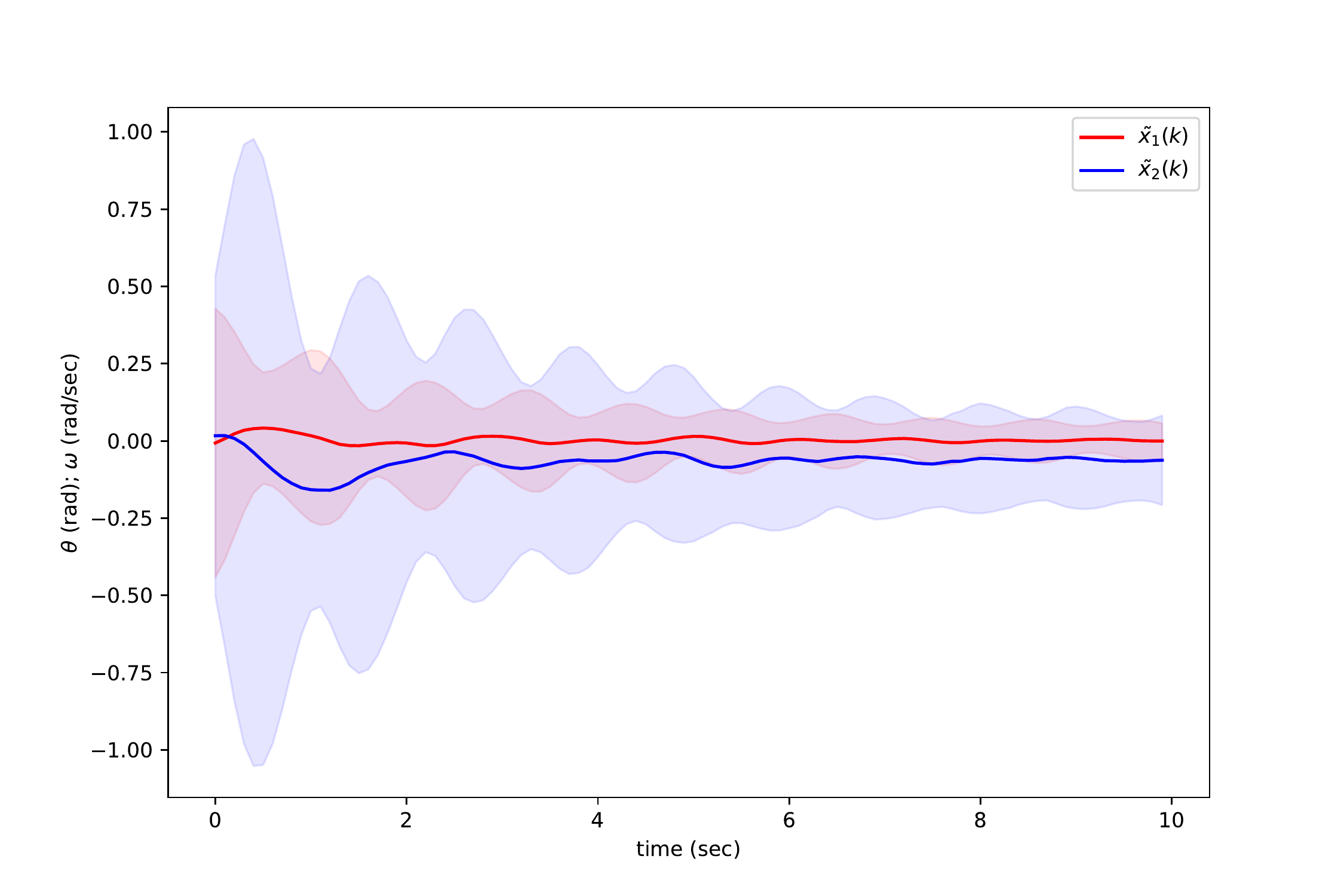}
		\end{minipage}}
		\subfloat[Exponential distribution]{
			\begin{minipage}{.33\textwidth}
				\centering
				\includegraphics[width=0.95\textwidth]{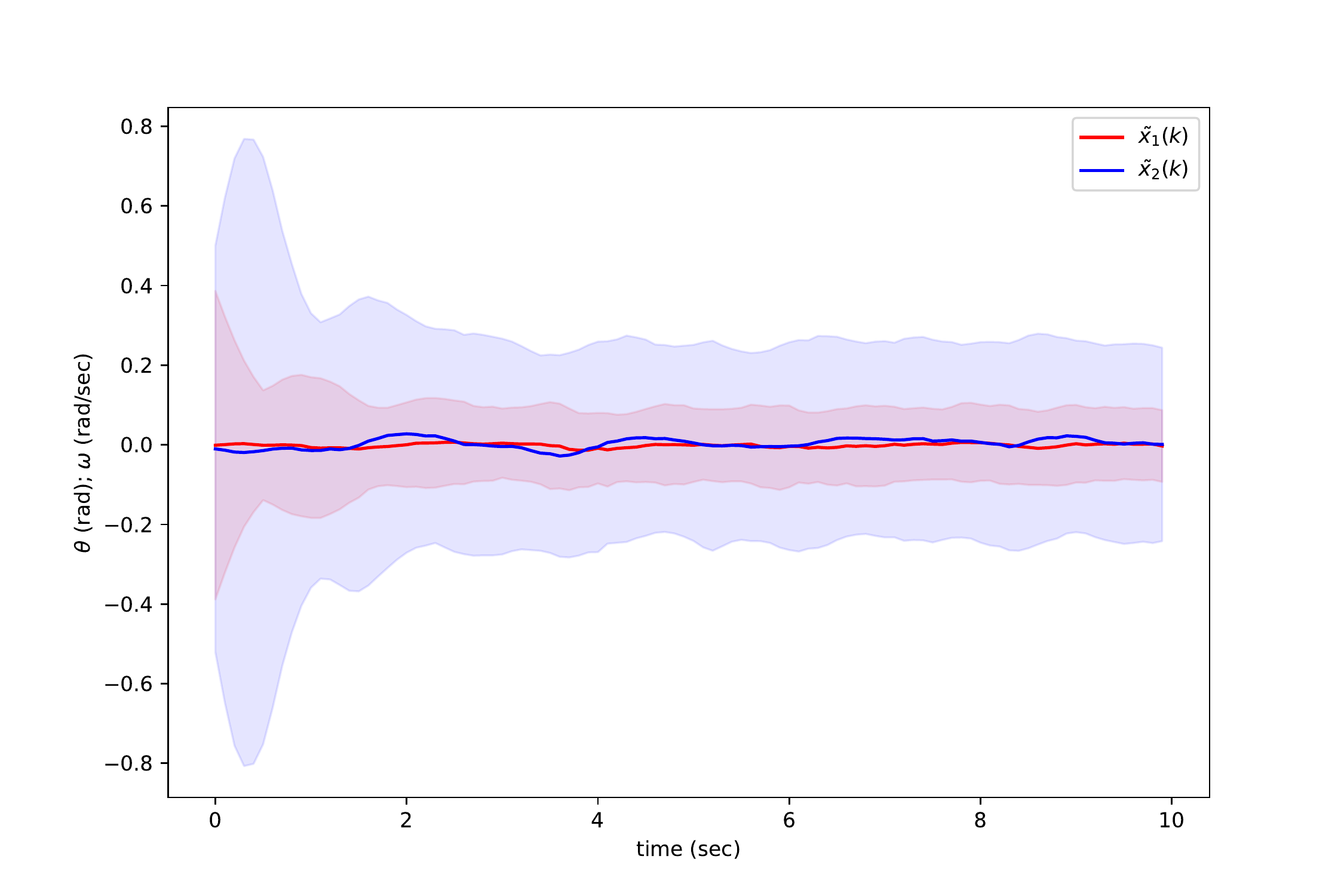}
		\end{minipage}}
		\subfloat[uniform distribution]{
			\begin{minipage}{.33\textwidth}
				\centering
				\includegraphics[width=0.95\textwidth]{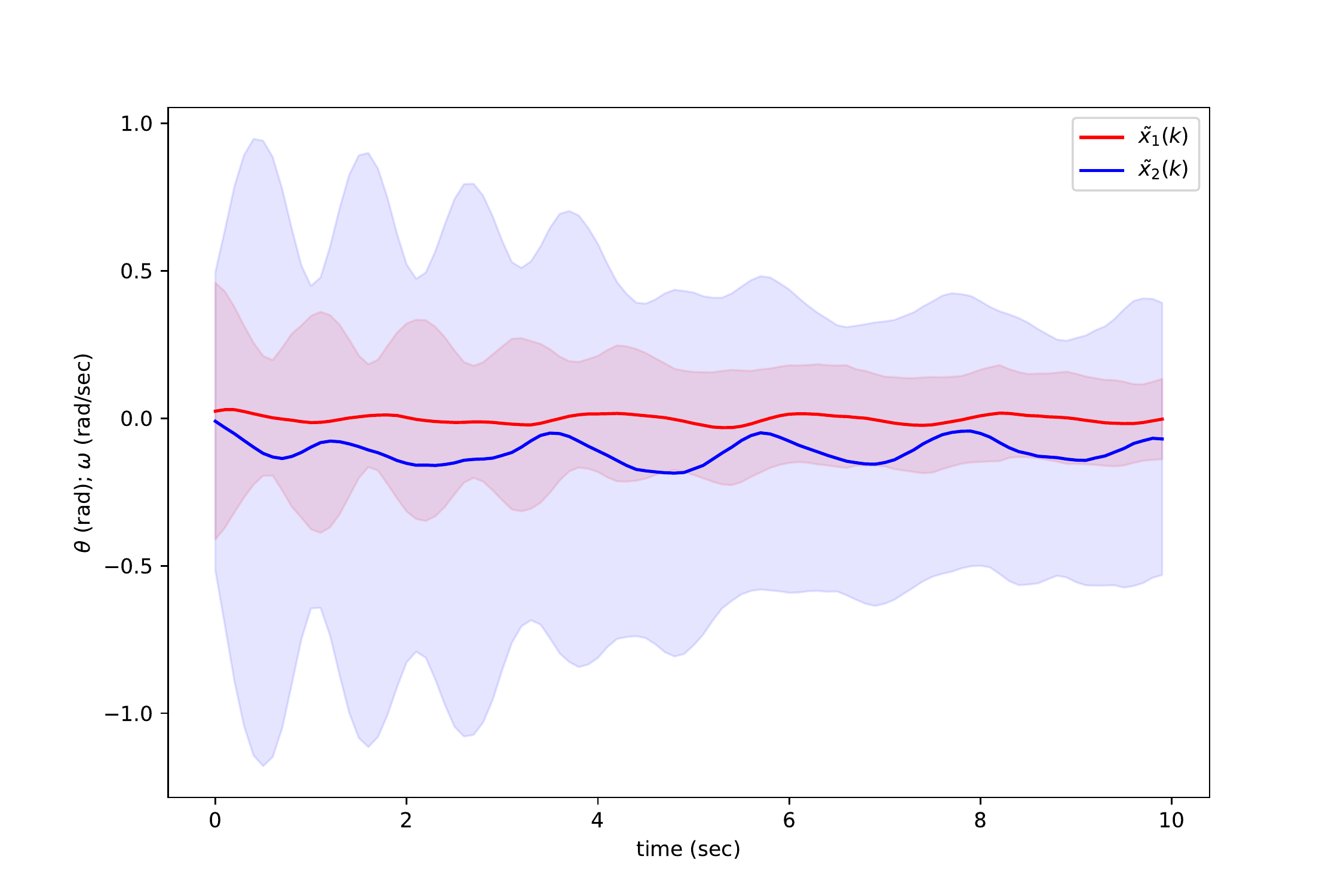}
		\end{minipage}}
		\caption{State estimate error of the LRLF with different kinds of measurement noise shown in Fig. \ref{fig:noise}. Solid line indicates the average estimate error $\tilde{x}_{1,k}$  (red) and $\tilde{x}_{2,k}$  (blue) and shadowed region for the 1-SD confidence interval. The results are obtained from 500 Monte Carlo runs.}
		\label{fig:lrlf-noise}
		\vspace{-0.5cm}
	\end{figure*}
	
	The proposed LRLF is compared with the EKF, UKF and PF. For each method, a number of $500$ Monte Carlo tests were run with random initial state given in \eqref{eq:initial_state} and \eqref{eq:initial_estimate}. The average value and SD of estimate error $\tilde{x}_{k}$ are illustrated in the left column of  Fig. \ref{fig:estimate_error} (see Figs. \ref{fig:estimate_error}(a, d, g, j, m)) for EKF, UKF, PF and LRLF. These comparisons show that only the PF with a larger number of particles ($10^4$) has comparable performance with the LRLF.

	To verify estimation performance of the LRLF for nonlinear systems with non-Gaussian noise, we tested the LRLF over three typical non-Gaussian distributions as shown in Fig. \ref{fig:noise}. Specifically, the measurement noise $v_k$ of the pendulum model in \eqref{eq1:pendulum}-\eqref{eq2:pendulum} is assumed to take three kinds of non-Gaussian probability distributions: 1) a Gaussian distribution $\mathcal{N}(0,0.01)$ truncated between the interval $[0,1]$; 2)  a uniform distribution at the interval $[-0.3, 0.3]$; or 3) an exponential distribution with the PDF
	$p(w_k)= \left\{{\begin{array}{cc}
		0.04\exp{(-0.04w_k)} & \text{if } w_k \geq 0 \\
		0 & \text{if } w_k < 0
		\end{array}}. \right.$
	A number of 500 Monte Carlo simulations with random initial state in \eqref{eq:initial_state} and \eqref{eq:initial_estimate} were test for each kind of measurement noise. The average value and SD of estimate error $\tilde{x}_{k}$ are illustrated in Fig. \ref{fig:lrlf-noise}. It is found that the LRLF still generates state estimate with bounded estimate errors under all the three kinds of non-Gaussian noises.

	\subsection{Robustness of the estimator}
	
	We considered the robustness of state estimators for two scenarios. (1) Against larger measurement noise: in the inference, the measurement noise variance increases from $0.01$ to $0.1$. (2) Against missing measurement: $y_k$ is sampled/contaminated randomly with Bernoulli distribution, in which $p(\text{missing})=p(\text{not missing})=0.5$. This means half of the measurements are not received/set to zeros during inference.
	
	\subsubsection{Larger measurement noise}
	
	Similarly, we run $500$ Monte Carlo tests with random initial states and initial state estimates given in \eqref{eq:initial_state}-\eqref{eq:initial_estimate} for each state estimator. The average value and variance of estimate error $\tilde{x}_{k}$ are illustrated in Figs.\ref{fig:estimate_error} (b), (e), (h), (k), (n). From Figs.\ref{fig:estimate_error} (b), (e), (h), it can be found that under a large measurement noise with variance 0.1, the estimates of the EKF, UKF and PF ($10^3$ particles) all diverge, while the estimate error of our proposed LRLF and PF ($10^4$ particles) increase compared with those under noise with variance $0.01$ but are both still bounded, as shown in Figs.\ref{fig:estimate_error} (k), (n). The LRLF produces comparable estimation performance to the computationally expensive PF with $10^4$ particles. It is worth pointing out that our proposed state estimator, though trained under noise with the variance of 0.01, still works well under a changing noise level, showing its robustness to the covariance shift of system noise.  From the left and middle column in Fig.~\ref{fig:estimate_error}, it is found that the EKF, UKF and PF ($10^3$ particles) have poor performance when measurement noise is large.
	
	\subsubsection{Missing measurement}
	
	Missing measurement is a common network-induced phenomenon and has been extensively investigated in state estimation for decades \cite{sinopoli2004kalman,wang2012h}. In our experiment, the measurement are sampled/contaminated randomly with a Bernoulli distribution, e.g., $p(\text{missing})=p(\text{no missing})=0.5$. The average value and variance of estimate error $\tilde{x}_{k}$ are illustrated in Fig. \ref{fig:estimate_error} (c), (f), (i), (l), (o). It is found that the EKF, UKF and PF with $10^3$ particles all diverge, LRLF has comparable performance as the PF with $10^4$ particles. These comparisons imply that LRLF is robust to randomly missing measurement.
	
	\begin{figure}
		\centering
		\includegraphics[width=0.45\textwidth]{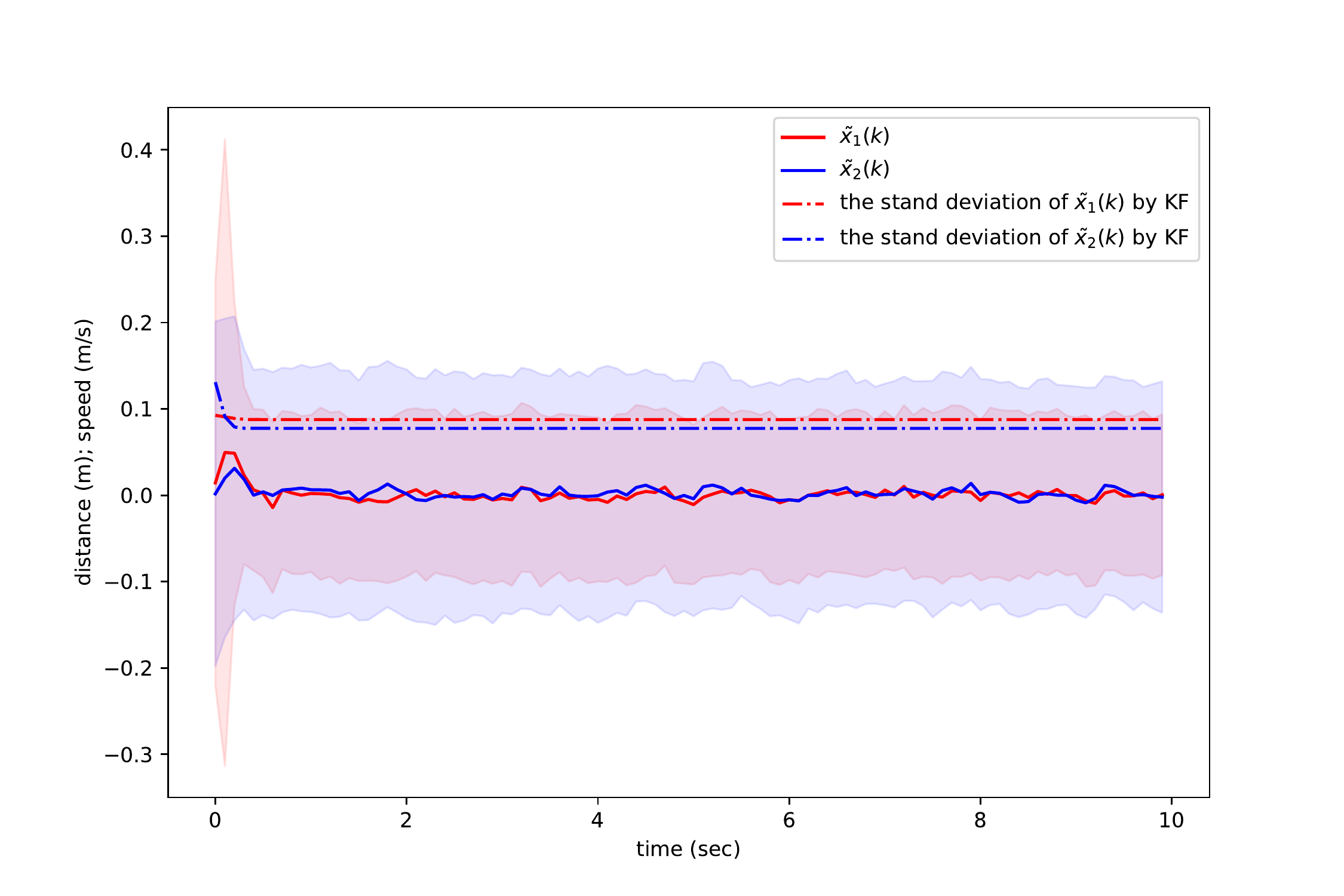}
		\caption{State estimate error of the LRLF vs the SD of state estimate of the Kalman filter. Solid line indicates the average estimate error $\tilde{x}_{1,k}$  (red) and $\tilde{x}_{2,k}$  (blue) and shadowed region for the 1-SD confidence interval of the LRLF. The results of LRLF are obtained with 500 Monte Carlo runs.}
		\label{fig:linear_systems}
		\vspace{-0.5cm}
	\end{figure}
	
	\subsection{A linear system case study}
	
	We also consider the linear systems, to compare with the classic Kalman filter which is known to be optimal for linear systems with Gaussian noise. The example is to track a vehicle running at a constant but unknown speed and subject to random noise \cite{sarkka2013bayesian}. The motion dynamics of the vehicle is given as follows:
	\begin{equation}
	\begin{bmatrix}
	x_{1,k+1}\\ x_{2,k+1}
	\end{bmatrix} = \begin{bmatrix}
	1&1\\0&1
	\end{bmatrix}\begin{bmatrix}
	x_{1,k}\\ x_{2,k}
	\end{bmatrix}+\begin{bmatrix}
	0\\1
	\end{bmatrix}w_{k}
	\end{equation}
	and the measurement model is given as follows:
	\begin{equation}
	y_{k}=\begin{bmatrix}
	1&0
	\end{bmatrix}x_{k}+v_{k}
	\end{equation}
	where $x_{1,k}$, $x_{2,k}$ are the distance and speed of the vehicle at time instant $k$, respectively. The Gaussian noise $w_{k} \sim \mathcal{N}(0,0.01)$ and $v_{k} \sim \mathcal{N}(0,0.02)$, and the initial state and the initial state estimate  are $x_0=\begin{bmatrix}
	0\\10
	\end{bmatrix}$, $\hat{x}_0 \sim \mathcal{N} \bigg(\begin{bmatrix}
	0\\10
	\end{bmatrix}, \begin{bmatrix}
	0.02&0\\0&0.03
	\end{bmatrix}\bigg)$, respectively.
	
	We use the same experimental setup as in the previous example of tracking pendulum. The estimate error of the LRLF is obtained from $500$ Monte Carlo simulations, as shown in Fig. \ref{fig:linear_systems}. The SD of state estimate error of the Kalman filter is plotted in Fig. \ref{fig:linear_systems} for comparison. It can be found that our proposed LRLF estimate both the distance and speed of the vehicle quite accurately, with slightly bigger estimate error than that of the optimal Kalman filter, which is expected since no state estimators can perform better than the optimal Kalman filter for linear stochastic systems with Gaussian noise.
	
	{\subsection{Airborne target tracking case study}
		We consider the scenario presented in \cite{liu2019particle}, where an unmanned aerial vehicle equipped with a gimballed camera to track a ground vehicle manoeuvring on a road section. Assume that the bearing-only camera platform is kept at the altitude $z^s = 100 m$ above the origin of the local coordinate  and the road is flat. The camera provides the azimuth angle ($\zeta$) and elevation angle ($\eta$) to the target with respect to the camera platform, as described by the following observation model:
		\begin{equation}\label{eq: ex3-obser}
		z_k = h(x_k)=\begin{bmatrix}
		\zeta_k \\ \eta_k
		\end{bmatrix}=\begin{bmatrix}
		\arctan_2(y_k,x_k)\\ \arctan_2(z^s, \sqrt{x_k^2+y_k^2})
		\end{bmatrix}+v_k
		\end{equation}
		where $x_k=\begin{bmatrix}
		s^x_k,s^y_k,\dot{s}^x_k,\dot{s}^y_k
		\end{bmatrix}^T$ is the target vehicle's state of position and speed components in $x$ and $y$ direction. The sensor noise $v_k$ is zero-mean Gaussian noise with the covariance $R=\begin{bmatrix}
		8&0\\0&0.002
		\end{bmatrix}$. 
		
		The target vehicle dynamics is described by a white noise acceleration motion model \cite{li2003survey}:
		\begin{equation}\label{eq: ex3-motion}
		x_{k+1}=\begin{bmatrix}
		1&0&T&0\\0&1&0&T\\0&0&1&0\\0&0&0&1
		\end{bmatrix}x_k+\begin{bmatrix}
		0.5T^2&0\\0&0.5T^2\\T&0\\0&T
		\end{bmatrix} w_k   
		\end{equation}
		where the sampling time $T=0.1 s$ and the process noise is zero-mean Gaussian noise with the covariance $Q=\begin{bmatrix}
		1&0\\0&1
		\end{bmatrix}$, the initial state and the initial state estimate are $x_0=\begin{bmatrix}
		98&0&0&10
		\end{bmatrix}^T$, $\hat{x}_0 \sim \mathcal{N} \bigg(\begin{bmatrix}
		98\\0\\0\\10
		\end{bmatrix}, \begin{bmatrix}
		25&0&0&0\\0&25&0&0\\0&0&1&0\\0&0&0&1
		\end{bmatrix}\bigg)$, respectively.
		
		We use the same experimental setup as in the previous example on tracking pendulum.  The state estimate of the  LRLF obtained from 500 Monte  Carlo simulations is shown in  Fig.~\ref{fig:example_3}. It shows that the proposed LRLF can track the target vehicle's state quite well.    
		
		\begin{figure}
			\centering
			\includegraphics[width=0.45\textwidth]{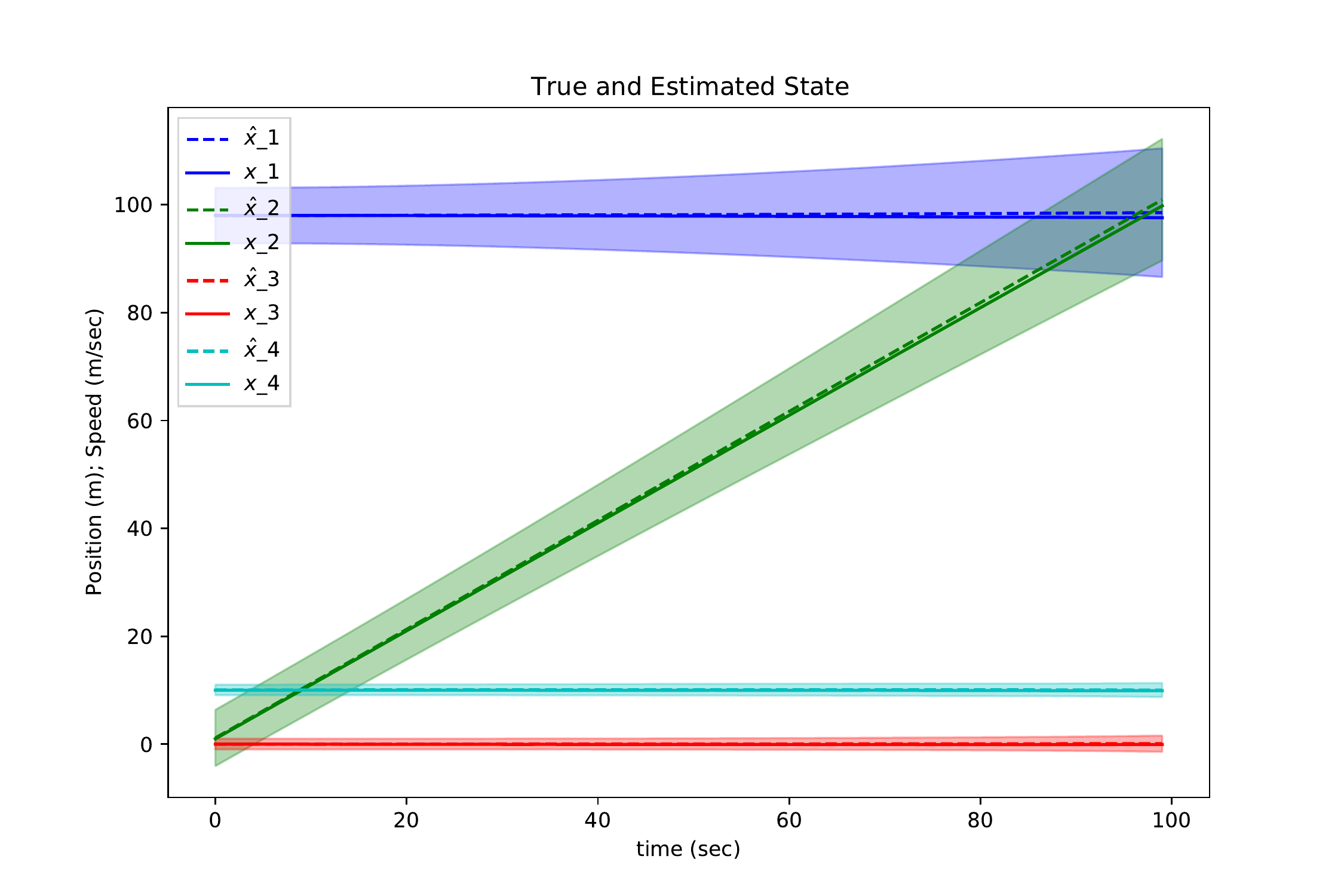}
			\caption{The true state of target vehicle and state estimate using the LRLF. Solid lines indicate true state and dashed lines the state estimate  and shadowed region for the 1-SD confidence interval of the LRLF. The results of LRLF are obtained with 500 Monte Carlo runs.}
			\label{fig:example_3}
			\vspace{-0.5cm}
		\end{figure}
	}
	
	\section{Conclusion and discussion}\label{sec:conclusion}
	This paper has combined Lyapunov's method in control theory and deep reinforcement learning to design state estimator for discrete-time nonlinear stochastic systems. We theoretically prove the convergence of bounded estimate error solely using the data in a model-free manner. In our approach, the filter gain is approximated by a deep neural network and trained offline. During inference, the learned filter can be deployed efficiently without extensive online computations. Simulation results show the superiority of our state estimator design method over existing nonlinear filters, in terms of estimate convergence even under some system uncertainties such as covariance shift in system noise and randomly missing measurements. {As initial research developing the  RL approach for state estimation, there are still quite a few issues that have not yet addressed in the work. For example, How to quantify uncertainty of the state estimate, i.e., the covariance of state estimate error?  What is the convergence bound of state estimate error with finite samples? They will be left as our future research directions.} 
	
	\bibliography{References}
	\bibliographystyle{IEEEtran}

\end{document}